\documentclass{article}

\usepackage{microtype}
\usepackage{dashrule} 
\usepackage{mdframed}
\usepackage{graphicx}
\usepackage{subfigure}
\usepackage{booktabs} 
\usepackage{multirow}
\usepackage{hyperref}
\usepackage{tabularx}
\newcolumntype{C}{>{\centering\arraybackslash}X}
\newcolumntype{P}[1]{>{\centering\arraybackslash}p{#1}}
\usepackage{boldline}
\usepackage{tcolorbox}

\usepackage[accepted]{icml2024}

\usepackage{amsmath}
\usepackage{amssymb}
\usepackage{mathtools}
\usepackage{amsthm}

\usepackage{xcolor,colortbl}
\usepackage{multirow}

\definecolor{blue}{RGB}{200,200,255}
\definecolor{green}{RGB}{200,255,200}
\definecolor{red}{RGB}{255,187,178}

\definecolor{blue2}{RGB}{57,61,218}
\definecolor{green2}{RGB}{52,163,58}
\definecolor{red2}{RGB}{208,88,88}

\definecolor{blueback}{RGB}{227,228,255}
\definecolor{blueout}{RGB}{0,3,112}

\usepackage[capitalize,noabbrev]{cleveref}

\theoremstyle{plain}
\newtheorem{theorem}{Theorem}[section]

\newtheorem{lemma}[theorem]{Lemma}

\theoremstyle{definition}
\newtheorem{definition}[theorem]{Definition}

\theoremstyle{remark}
\newtheorem{remark}[theorem]{Remark}

\newcommand{\floor}[1]{\lfloor #1 \rfloor}

\usepackage[textsize=tiny]{todonotes}

\icmltitlerunning{Submission and Formatting Instructions for ICML 2024}

\begin{document}

\twocolumn[
\icmltitle{An Analysis of Linear Time Series Forecasting Models}



\icmlsetsymbol{equal}{*}

\begin{icmlauthorlist}
\icmlauthor{William Toner}{sch,comp2}
\icmlauthor{Luke Darlow}{comp2}

\end{icmlauthorlist}

\icmlaffiliation{sch}{Department of Informatics, University of Edinburgh, Edinburgh}
\icmlaffiliation{comp2}{SIR Lab, Huawei Research Centre, Edinburgh}

\icmlcorrespondingauthor{William Toner}{w.j.toner@sms.ed.ac.uk}
\icmlcorrespondingauthor{Luke Darlow}{luke.darlow1@huawei.com}

\icmlkeywords{Machine Learning, Time Series}

\vskip 0.3in
]



\printAffiliationsAndNotice{}  

\begin{abstract}
Despite their simplicity, linear models perform well at time series forecasting, even when pitted against deeper and more expensive models. A number of variations to the linear model have been proposed, often including some form of feature normalisation that improves model generalisation. In this paper we analyse the sets of functions expressible using these linear model architectures. In so doing we show that several popular variants of linear models for time series forecasting are equivalent and functionally indistinguishable from standard, unconstrained linear regression. We characterise the model classes for each linear variant. We demonstrate that each model can be reinterpreted as unconstrained linear regression over a suitably augmented feature set, and therefore admit closed-form solutions when using a mean-squared loss function. We provide experimental evidence that the models under inspection learn nearly identical solutions, and finally demonstrate that the simpler closed form solutions are superior forecasters across 72\% of test settings.

\end{abstract}
\section{Introduction}

Time series forecasting is a crucial challenge across a wide array of domains, where accurate predictions of the future are essential. Key areas such as finance, meteorology, healthcare, cloud infrastructure, and traffic flow management rely heavily on forecasting for decision-making and strategic planning \citep{wu2021autoformer, lai2018modeling, sloss2019metrics,taylor2018forecasting, darlow2023foldformer, Joosen2023How}. This has led to significant research efforts to develop effective forecasting models. Deep learning has transformed many fields, most notably in computer vision and language processing, superseding simpler classical models. Following these successes, deep learning has seen increasing usage in time series forecasting. In particular transformer models have been adapted for broader time series forecasting applications \citep{nie2022time,zhou2021informer,liu2021pyraformer,wu2021autoformer,liu2023itransformer,anonymous2024dam}.

\paragraph{Linear models for forecasting}
Despite the advantages offered by deep learning, its application to time series forecasting has encountered unique challenges. Recent studies have shown that the performance benefits of deep models for forecasting are often marginal when compared to simpler linear models \cite{zeng2023transformers,li2023revisiting}. Linear models are also appealing due to their simplicity, explainability, and efficiency. This is particularly relevant in industries where forecasting models are queried frequently and/or involve high-resolution data, such as in cloud resource allocation \citep{Joosen2023How,darlow2023foldformer}. This has spurred a growing interest in refining linear models: several variants of linear time series forecasting models have emerged, each purporting superiority owing to some architectural difference (summary in Section \ref{sec:related}). We show in this research that many of these popular, and often high-performing, linear models are essentially equivalent. By this we mean that the parametric families of functions which they describe, are equal (up to choice of data normalisation). The convexity of least-squares linear regression makes this a significant finding since it implies that all these models should converge to the same optima, given a suitable optimiser.

\paragraph{Outline and Contributions}

In this paper, we delve into the mathematics of several well-known linear time series forecasting models. We fully characterise the set of functions which are expressible using each architecture. We show, somewhat remarkably, that they are all essentially equivalent: corresponding either to unconstrained or weakly constrained (via feature augmentation) linear regression. The convexity of least-squares linear regression suggests that the behaviour of these models should therefore be virtually indistinguishable. We provide experimental evidence which supports this hypothesis, showing that, in practice, all models tend to the same optima. Furthermore, we show that the closed form solution to least-squares linear regression performs either comparably or better than those trained by gradient descent. Our contributions are:

\begin{enumerate}
    \item Mathematical proofs that several popular linear models for time series forecasting are essentially identical.
    \item Experimental evidence that each model indeed tend to the same solution when trained on the same data, differing only in the bias parameter.
    \item Quantitative evidence that closed form ordinary least squares (OLS) solutions are typically superior to existing models trained using stochastic gradient descent.
\end{enumerate}

The goal of this paper is to provide a, much-needed, in-depth mathematical analysis of several popular linear time-series forecasting models. We aim to demonstrate that, from a functional and performance point of view, these models are not substantially different to each other and amount to weakly constrained linear regression. 

\section{Related Work}\label{sec:related}

\citet{zeng2023transformers} asked the important question of whether the transformer architecture \citep{vaswani2017attention} had utility for time series forecasting. Their work introduced two models, namely \textbf{DLinear} (Section \ref{sec:dlinear}) and \textbf{NLinear} (Section \ref{sec:nlinear}), that have become widely used baselines for other research in time series forecasting \citep{anonymous2024dam, nie2022time, liu2023itransformer}. Their work served to show that linear models are comparable, and sometimes better, than complex transformer architectures.

Reversible instance normalisation \citep{kim2021reversible} (RevInv) is a feature normalisation technique that typically improves time series forecasting. It operates by standardising input features (zero mean, unit standard deviation) before passing these through a given model, and reversing this standardisation process as a final step (with an optional learnable affine transformation). We unpack the mathematics of how various modes of instance normalisation constrain the underlying model class in Section \ref{sec:invertible_norms}.

\citet{li2023revisiting} revisited long-term time series forecasting by exploring the impact of RevInv and channel independence (CI). CI for linear models implies learning distinct models for each variate in a given dataset. They proposed \textbf{RLinear} -- a linear mapping that uses RevInv -- and tested the impact of CI, showing how for some datasets (usually with a higher number of channels and/or complexity) CI improves generalisation. We define RLinear in Section \ref{sec:rlinear}.

\citet{xu2023fits} recently proposed \textbf{FITS}, a linear time series model that operates in frequency space and includes an optional high-frequency filtering component to reduce the model footprint. FITS first computes the real discrete Fourier transform (RFT), applies a complex linear map, and inverts the result back into the time domain. We define FITS in Section \ref{sec:fits}. The performance of FITS is impressive, obtaining at or near state of the art (SoTA) under its optimal hyperparameter settings.

\section{Analysis of Linear Time Series Forecasting Models}\label{sec:model_analysis}
For the purpose of this paper we refer to a `model class' as the parametric set of functions induced by a model architecture. For example, a single layer linear neural network with no hidden layer has the model class $\vec{x}\mapsto A\vec{x}+\vec{b}$, where the dimensions of $A$ and $\vec{b}$ are as appropriate. We call this `\textbf{Linear}' for the remainder of this paper. In this section we define the task of forecasting with a linear model. We then analyse the widely used DLinear (Section \ref{sec:dlinear}) and the recent SoTA FITS architectures (Section \ref{sec:fits}). We prove mathematically that these models are equivalent to linear regression in that they have the same model class.

We then define and discuss several invertible data normalisation strategies employed for time series forecasting in Section \ref{sec:invertible_norms}. These normalisation strategies yield additional linear model variants, namely RLinear, NLinear, and FITS+IN (i.e., FITS with instance normalisation, as per \citet{xu2023fits}). We show how each choice of feature normalisation restricts the model class. This allows us to categorise all linear model variants into only 3 similar but distinct classes. 

\subsection{Notation}
The following notations are used throughout this paper:\vspace{-2mm}
\begin{itemize}
\itemsep0em 
    \item $L$: Context length (time steps in the input sequence).
    \item $c$: Number of channels (distinct time series).
    \item $T$: Forecast horizon (future time steps predicted).
    \item $\vec{x}$: Context vector (historical data), $\vec{x} \in \mathbb{R}^{L \times c}$.
    \item $\vec{y}$: Target vector (values to be predicted), $\vec{y} \in \mathbb{R}^{T \times c}$.
\end{itemize}

The models we look at in Section~\ref{sec:model_analysis} do not explicitly use cross-channel information in their predictions. By this we mean that the $i$\textsuperscript{th} channel of the target is predicted only using the $i$\textsuperscript{th} channel of the context. Therefore, for improved clarity, we consider the case of $c=1$ (univariate), with $\vec{x} \in \mathbb{R}^L$.

\begin{definition}[Forecast Model and Model Class]\label{def:model_class}
    A \textbf{forecast model} is a function \( f: \mathbb{R}^L \rightarrow \mathbb{R}^T \) that generates a forecast \( \vec{y} \) from a given input, or context vector \( \vec{x} \). The collection of such forecast models forms a \textbf{Model Class}, denoted as \( M \).

    \textbf{Example}: In the case of \textbf{(unconstrained) Linear Regression}, the model class \( M(\text{Linear}) \) consists of all functions of the form \( \vec{x} \mapsto W\vec{x} + \vec{b} \), where \( W \) is a weight matrix from \( \mathbb{R}^{T \times L} \) and \( \vec{b} \) is a bias vector from \( \mathbb{R}^T \). If there are specific limitations or conditions applied to the values of \( W \) and \( \vec{b} \), we term this as \textbf{constrained Linear Regression}.
\end{definition}

\subsubsection{DLinear}\label{sec:dlinear}
\textbf{DLinear Model Definition}: Let $\vec{x}\in \mathbb{R}^{L}$ be a context vector. DLinear works by decomposing $\vec{x}$ into a `trend' and `seasonal' components. The trend component is defined by taking a moving average of the components of $\vec{x}$. The seasonal component is given by the residual $\vec{x}_{\text{seasonal}}  \coloneqq \vec{x} - \vec{x}_{\text{trend}}$. The moving average is padded so that it preserves the dimensionality of $x$. One then takes $\vec{x}_{\text{seasonal}}$ and $\vec{x}_{\text{seasonal}}$ and passes these though separate learnable linear layers.

 \begin{lemma}[DLinear Model Class]\label{lemma:dlinear}
 Let $M(\text{DLinear})$ denote the DLinear model class, i.e. the set of functions $f:\mathbb{R}^L\rightarrow\mathbb{R}^T$ which can be represented as a DLinear model. $M(\text{DLinear})$ is precisely equal to the space of affine linear functions. That is, all functions of the form $Ax + \vec{b}$ may be expressed as a DLinear model and vice versa. 
 \end{lemma}
 \begin{proof}
 Following our definition, any DLinear model can be written as $B\vec{x}_{\text{seasonal}} + C\vec{x}_{\text{trend}} + \vec{c}+\vec{d}$ where $B,C\in \mathbb{R}^{T\times L}$, $\vec{c},\vec{d}\in \mathbb{R}$ are the weight matrices and biases of DLinear's two linear layers.  This can be expressed as $B(\vec{x}- \vec{x}_{\text{trend}}) + C(\vec{x}_{\text{trend}}) + \vec{c}+\vec{d} = B(\vec{x}- D\vec{x}) + C(D\vec{x}) + \vec{c}+\vec{d} = (B-BD+CD)\vec{x} + \vec{c}+\vec{d}$ where $D$ is the (square) matrix corresponding to a padded moving average (See Appendix~\ref{App:futher_proofs} for an explanation). Thus we have shown that any DLinear model may be expressed in the form $A\vec{x} + \vec{b}$. It remains to show the converse, that is, any affine linear map is expressible in the form of a DLinear model. 
 
 Let $A\vec{x} + \vec{b}$ be some arbitrary affine linear map. We claim that $A\vec{x} + \vec{b}$ can be expressed in the form $(B-BD+CD)\vec{x} + \vec{c}+\vec{d}$. By setting e.g. $\vec{c}=\vec{b}$, $\vec{d}=0$ we match the bias terms. By setting $B=C=A$ we match the weight matrices  
  \end{proof}\vspace{-3mm}
\vspace{-0.8mm}
\begin{tcolorbox}[boxsep=3pt,left=0pt,right=0pt,top=0pt,bottom=0pt,halign=center,colback=blueback,colframe=blueout]
\small
\textbf{\boldmath$M(\text{DLinear}) = M(\text{Linear})$}
\end{tcolorbox}\vspace{-2mm}

\subsubsection{FITS}\label{sec:fits}
\textbf{FITS Model Definition:} Let $\vec{x}\in \mathbb{R}^L$ be a context vector. FITS applies the Real (discrete) Fourier Transform (RFT) to $\vec{x}$. This maps $\vec{x}$ to a complex vector of length $\floor{L/2}+1$. Next one applies a learnable complex linear map with output dimension $\floor{(L+T)/2}+1$. After this one applies the inverse RFT to map to $\mathbb{R}^{L+T}$.

\textbf{Remark:} As proposed by \citet{xu2023fits}, FITS optionally includes a low-pass filter (LPF) to discard high frequencies components. Our initial experiments showed that utilising a LPF results in a degradation in performance -- this is confirmed by analysing the settings of FITS that yield high-performance. Thus, we analyse FITS without the LPF.

\textbf{Remark}: Unlike other models, FITS outputs both a forecast and a reconstruction of the context vector. The forecast may be obtained by discarding the first $L$ components output by the model. 

\begin{theorem} [FITS Model Class] \label{theorem:fits}
Let $M(\text{FITS})$ denote the FITS model class, i.e. the set of functions $f:\mathbb{R}^L\rightarrow\mathbb{R}^{T}$ which can be represented as a FITS model. When $L\geq T-2$, $M(\text{FITS})$ is precisely equal to the space of affine linear functions $A\vec{x} + \vec{b}$. 
\end{theorem}

Proving Theorem~\ref{theorem:fits} is somewhat involved. Importantly, as a combination of a Fourier transform, a complex linear map, an inverse Fourier transform, FITS is a composition of linear maps and is therefore expressible in the form $A\vec{x}+\vec{b}$. The proof in Appendix~\ref{sec:fitsProof} shows when $L\geq T-2$ that both $A$ and $\vec{b}$ are entirely unconstrained. This is significant since all the settings in \citet{xu2023fits} utilise a context larger or equal to the prediction horizon $T$.\vspace{-1mm}

\begin{tcolorbox}[boxsep=3pt,left=0pt,right=0pt,top=0pt,bottom=0pt,halign=center,colback=blueback,colframe=blueout]
\small
\textbf{\boldmath$M(\text{DLinear}) = M(\text{Linear}) = M(\text{FITS})$}\vspace{1mm} \\(when $L\geq T-2$ and without a low-pass filter)
\end{tcolorbox}\vspace{-1mm}

\subsection{Invertible Data Normalisations}\label{sec:invertible_norms}
Invertible instance-wise feature normalisation has been recently adopted for time-series forecasting. `Instance normalisation' (in the context of time series) was proposed by \citet{kim2021reversible}. In this section we cover three such mechanisms: Instance Norm (IN), Reversible Instance Norm (RevIN), and \textit{NowNorm} (NN) which is the name we give to the normalisation scheme implemented by NLinear. For clarity, RevIN and IN are identical except for the learnable affine mapping of RevIN -- we mark this distinction because the optional learnable affine map is often not used (e.g., FITS). We look at how each normalisation restricts the model class when used in conjunction with linear models.

\subsection{Instance Norm}

\begin{definition}[Instance normalisation]\label{def:instance_norm}
Given a context vector $\vec{x}$ and a target vector $\vec{y}$, \textbf{instance normalisation} (IN) for each data instance involves normalizing $\vec{x}$ by its mean $\mu(\vec{x})$ and standard deviation $\sigma(\vec{x} )$, applying a model $f$ on the normalized $\vec{x}'$, and inversely transforming the prediction $\hat{y}$ back to the original scale. Formally, this is expressed as:
\begin{align*}
    \vec{x}' &= \frac{\vec{x} - \mu(\vec{x} )}{\sigma(\vec{x} ) + \varepsilon}, \\
    \hat{y} &= f(\vec{x}'), \\
    \hat{y}_{\text{out}} &= \hat{y} \cdot (\sigma(\vec{x} )+\epsilon) + \mu(\vec{x} ),
\end{align*}
where $\varepsilon$ is a small constant for numerical stability.
\end{definition}

\begin{lemma}[Linear+IN]\label{lemma:ilinear}
Let $M(\text{ILinear})$ represent the set of forecast models that can be expressed as a linear layer combined with instance normalization (Definition~\ref{def:instance_norm}). $M(\text{ILinear})$ is equal to the set of functions $f: \mathbb{R}^L \rightarrow \mathbb{R}^T$ expressible in the form $\tilde{A}\vec{x} + \vec{b}\sigma(\vec{x})$. $\tilde{A}$ is a matrix with each row summing to $1$, and $\sigma(\vec{x})$ is the standard deviation of $\vec{x}$. 
\end{lemma}

\begin{proof}
Let $\vec{x}\in \mathbb{R}^L$ be a context vector. Let $f$ be a forecast model obtained by applying a linear layer after IN. If $A,\vec{b}$ are the weight matrix and bias of the linear layer then we have $f(\vec{x}) = \vec{\mu}(\vec{x}) + \sigma(\vec{x})(A(\frac{\vec{x}-\vec{\mu}(\vec{x})}{\sigma(\vec{x})}) + \vec{b})$. Here $\vec{x}-\vec{\mu}(\vec{x})$ denotes the subtraction of the mean $\mu(\vec{x})$ from every component of $\vec{x}$. We can expand this expression out to obtain simply $\vec{\mu}(\vec{x}) + A(\vec{x}-\vec{\mu}(\vec{x})) + \sigma(\vec{x})\vec{b}$. The $T-$dimensional vector $\vec{\mu}(\vec{x})$ of means can be written as a matrix multiplication $B_T\vec{x}$ where $B_m$ denotes a matrix of shape $m\times L$ populated exclusively by $\frac{1}{L}$'s. Using this notation we have:
\begin{align*}
f(\vec{x}) =& \vec{\mu}(\vec{x}) + A(\vec{x}-\vec{\mu}(\vec{x}))  + \sigma(\vec{x})\vec{b}\\
=& (B_T+A-AB_L)\vec{x}  + \sigma(\vec{x})\vec{b}
\end{align*}
Thus $f$ can be written in the form $\tilde{A}\vec{x}+b\sigma(\vec{x})$, it remains only to demonstrate that $B_T+A-AB_L$ satisfies the condition that the rows sum to one. And, conversely that any matrix of shape $T\times L$ whose rows sum to one can be written in this form. Begin by noting that $AB_L$ can be written as:

\begin{frame}
\footnotesize
\setlength{\arraycolsep}{2.5pt}
\medmuskip = 1mu 
\begin{align*}
AB_L = \begin{bmatrix}
    A_{11}       & A_{12} & \dots & A_{1L} \\
    A_{21}       & A_{22} & \dots & A_{2L} \\
    \hdotsfor{4} \\
   A_{L1}       & A_{L2} & \dots & A_{TL} 
\end{bmatrix}
\begin{bmatrix}
    1/L       & 1/L & \dots & 1/L \\
    1/L       & 1/L & \dots & 1/L \\
    \hdotsfor{4} \\
   1/L       & 1/L & \dots & 1/L 
\end{bmatrix}
\\= \begin{bmatrix}
    \frac{1}{L}\sum_{i=1}^L A_{1i}       & \frac{1}{L}\sum_{i=1}^L A_{1i} & \dots & \frac{1}{L}\sum_{i=1}^L A_{1i} \\
    \frac{1}{L}\sum_{i=1}^L  A_{2i}       & \frac{1}{L}\sum_{i=1}^L  A_{2i} & \dots & \frac{1}{L}\sum_{i=1}^L  A_{2i} \\
    \hdotsfor{4} \\
   \frac{1}{L}\sum_{i=1}^L  A_{Li}  & \frac{1}{L}\sum_{i=1}^L  A_{Li} & \dots & \frac{1}{L}\sum_{i=1}^L  A_{Li} 
\end{bmatrix}
\end{align*}
Therefore the $ij$\textsuperscript{th} element of $(B_T+A-AB_L)$ may be written as $\frac{1}{L} + A_{ij}-\frac{1}{L}\sum_{k=1}^L A_{ik}$. It follows that the sum of row $i$ of $(B_T+A-AB_L)$
\begin{align*}
    =& \sum_{j=1}\big(\frac{1}{L} + A_{ij}-\frac{1}{L}\sum_{k=1}^L A_{ik}\big) \\ 
    =& 1 + \sum_{j=1}^LA_{ij} - \frac{L}{L}\sum_{k=1}^L A_{ik} \\
    =& 1 .
\end{align*}
\end{frame}
Conversely we wish to show that any $T$ by $L$ matrix $C$ whose rows sum to one can be written in the form $(B_T+A-AB_L)$. Let $C$ be such a matrix. One may easily show that in fact $C$ may be expressed as $C=B_T+C-CB_L$, thus we may let $A=C$. 
\end{proof}



\subsection{Reversible Instance Normalisation}
A second more general form of data normalisation is known as Reversible Instance Norm (RevIN) \cite{kim2021reversible}. This normalisation is designed to allow a forecasting model to handle shifts in the temporal distribution over time. \citet{li2023revisiting} showed that a simple linear model using RevIN is able to outperform most deep models on standard datasets. 

\begin{definition}[Reversible Instance Normalisation]\label{def:rin}
Given a context vector $\vec{x}$ and a target vector $\vec{y}$, \textbf{Reversible Instance Normalisation} (RevIN) for each data instance involves a two-step normalization process. First, $\vec{x}$ is normalized by its mean $\mu(\vec{x})$ and standard deviation $\sigma(\vec{x})$. Subsequently, an affine transformation with parameters $\alpha$ and $\beta$ is applied, followed by the application of a forecasting model $f$ on the transformed $\vec{x}'$. The process is then reversed to retrieve the prediction in the original scale. Formally, this is expressed:
\begin{align*}
\vec{x}' &= \frac{\vec{x} - \mu(\vec{x})}{\sigma(\vec{x}) + \varepsilon}, \\
\vec{x}'' &= \frac{\vec{x}' - \beta}{\alpha},  \\
\hat{y} &= f(\vec{x}''), \\
\hat{y}' &= \alpha\hat{y} + \beta, \\
\hat{y}_{\text{out}} &= \hat{y}' \cdot (\sigma(\vec{x}) + \varepsilon) + \mu(\vec{x}).
\end{align*}
\end{definition}

\subsubsection{RLinear}\label{sec:rlinear}
RLinear is a linear model using RevIN \citep{li2023revisiting}.

\begin{lemma} [RLinear]
Let $M(\text{RLinear})$ denote the RLinear model class, i.e. the set of functions $f:\mathbb{R}^L\rightarrow\mathbb{R}^T$ which can be represented as an  RLinear model. $M(\text{RLinear})$ is precisely equal to the space of functions $\tilde{A}\vec{x} + \vec{b}\sigma(\vec{x})$ where the rows of $\tilde{A}$ each sum to $1$ and where $\sigma(\vec{x})$ denotes the standard deviation of the context vector $\vec{x}$. 
\end{lemma}
\begin{proof}
Let $f$ be arbitrary RLinear model (i.e., a forecast model obtainable by the composition of a linear layer and reversible instance norm). If $A,\vec{c}$ are the weight matrix and bias of the linear layer then 
\begin{align*}
    f(\vec{x}) = \mu(\vec{x})+\sigma(\vec{x})\left(\beta +\alpha (AR(\vec{x}) + c)\right)\\
    \text{where } R(\vec{x}) := \frac{1}{\alpha}\left(\frac{\vec{x}-\mu(\vec{x})}{\sigma(\vec{x})}-\beta\right)
\end{align*}
We can expand this out to obtain the following
\begin{frame}
\footnotesize
\setlength{\arraycolsep}{2.5pt}
\medmuskip = 1mu 
\begin{align*}
    f(\vec{x})=(\mu(\vec{x}) +A\vec{x} -A\mu(\vec{x})) + \beta\sigma(\vec{x}) +\alpha c\sigma(\vec{x})-A\beta\sigma(\vec{x}).
\end{align*}
\end{frame}
As per the proof on Lemma~\ref{lemma:ilinear} we can write the vector of means $\mu(\vec{x})$ as $B_T\vec{x}$ where $B_T$ is a $T\times L$ matrix populated by $\frac{1}{L}$'s. Therefore $f(\vec{x})$ can be expressed as 
\begin{align*}
   f(\vec{x}) =& \tilde{A}\vec{x} + \sigma(\vec{x})\vec{b}\\
    \text{where } \tilde{A} =& B_T +A -AB_L\\
    \text{and } \vec{b}=& \beta+\alpha c-A\beta
\end{align*}
As in the proof of Lemma~\ref{lemma:ilinear}, $B_T +A -AB_L$ are precisely the set of matrices where each row sums to one. It is left therefore to demonstrate that $\vec{b}$ can be any vector in $\mathbb{R}^T$. Since we are free in our choice of $\beta, \alpha, c$ then we can let $\beta =0, \alpha =1$ and $\vec{c}$ be any desired arbitrary vector in $\mathbb{R}^T$. This concludes the proof.
\end{proof}
\vspace{-1mm}
\begin{tcolorbox}[boxsep=3pt,left=0pt,right=0pt,top=0pt,bottom=0pt,halign=center,colback=blueback,colframe=blueout]
\small
\textbf{IN} and \textbf{RevIN} impose the constraints: (1) the rows of the weight matrix must sum to 1; (2) the bias is scaled by the standard deviation of the instance.
\end{tcolorbox}

\subsection{NowNorm}

\begin{definition}[Now-Normalisation]\label{def:end_norm}
Given a context vector $\vec{x}$ and a target vector $\vec{y}$, \textbf{NowNorm} (NN) involves normalising the context so that $x_{L}$, the most-recent value of $\vec{x}$, is zero. Explicitly; $\vec{x}_{\text{norm}} := \vec{x} - (x_L, x_L, \ldots, x_L)$. Next we apply a forecasting model $f$ on the normalized $\vec{x}_{\text{norm}}$, before adding $x_L$ back on to each component of the output. Formally, this is expressed as:
\begin{align*}
    \vec{x}_{\text{norm}} &= \vec{x} - x_L, \\
    \hat{y} &= f(\vec{x}_{\text{norm}}), \\
    \hat{\vec{y}}_{\text{out}} &= \hat{y} + x_L.
\end{align*}
\end{definition}

\subsubsection{NLinear}\label{sec:nlinear}
Nlinear is a linear model using NN \citep{zeng2023transformers}.

\begin{lemma}[NLinear] \label{lemma:nlinear}
    Let $M(\text{NLinear})$ denote the NLinear model class, i.e. the set of functions $f:\mathbb{R}^L\rightarrow\mathbb{R}^T$ which can be represented as a NLinear model. $M(\text{NLinear})$ is precisely equal to the space of linear functions $\tilde{A}\vec{x} + \vec{b}$ where the rows of $\tilde{A}$ each sum to $1$. 
\end{lemma}
\begin{proof}
     By Definition~\ref{def:end_norm}, an NLinear model can be written
     \begin{align}
         (x_L, x_L, \ldots, x_L) + A\vec{x}_{\text{norm}} +\vec{b} \label{eqn:NLinear}
     \end{align}
      Where $A,\vec{b}$ are the weight matrix and bias terms of NLinear's linear layer and $\vec{x}_{\text{norm}}$ is the normalised context vector. If we let $B_m$ denote an $m$ by $m$ matrix with 1's in the final column and zeros elsewhere. Then we can write Equation~\ref{eqn:NLinear} equivalently as $B_T\vec{x} + A(\vec{x} - B_L\vec{x}) +\vec{b} = (B_T+ A-AB_L)\vec{x} +\vec{b}$.  We claim that the rows of the matrix $B_T+ A-AB_L$ sum to one. Begin by noting that $AB_L$ has the following form:

\begin{frame}
\footnotesize
\setlength{\arraycolsep}{2.5pt}
\medmuskip = 1mu 
\begin{align*}
AB_L = \begin{bmatrix}
    A_{11}       & A_{12} & \dots & A_{1L} \\
    A_{21}       & A_{22} & \dots & A_{2L} \\
    \hdotsfor{4} \\
   A_{L1}       & A_{L2} & \dots & A_{TL} 
\end{bmatrix}
\begin{bmatrix}
    0       & 0 & \dots & 1 \\
    0       & 0 & \dots & 1 \\
    \hdotsfor{4} \\
   0       & 0 & \dots & 1 
\end{bmatrix}
\\= \begin{bmatrix}
    0       & 0 & \dots & \sum_{i=1}^L A_{1i} \\
    0       & 0 & \dots & \sum_{i=1}^L  A_{2i} \\
    \hdotsfor{4} \\
   0       & 0 & \dots & \sum_{i=1}^L  A_{Li} 
\end{bmatrix}
\end{align*}
\end{frame}

Therefore $B_T + A - AB_L$ may be expressed as follows:

\begin{frame}
\footnotesize
\setlength{\arraycolsep}{2.5pt}
\medmuskip = 1mu 
\begin{align}
    \begin{bmatrix}
    A_{11}       & A_{12} & \dots & 1-\sum_{i=1}^{L-1}  A_{1i} \\
    A_{21}       & A_{22} & \dots & 1-\sum_{i=1}^{L-1} A_{2i} \\
    \hdotsfor{4} \\
   A_{L1}       & A_{L2} & \dots & 1-\sum_{i=1}^{L-1} A_{Li} \label{eqn:row_sum}
\end{bmatrix}
\end{align}
\end{frame}

It is clear to see that the rows of this matrix sum to one as claimed. Moreover, any matrix whose rows sum to 1 may be written as Equation~\ref{eqn:row_sum}. Since the bias $\vec{b}$ is unconstrained then we conclude our proof. 
\end{proof}

\vspace{-2mm}
\begin{tcolorbox}[boxsep=3pt,left=5pt,right=5pt,top=5pt,bottom=5pt,colback=blueback,colframe=blueout]
\small
\textbf{NowNorm} imposes the same weight matrix constraint as \textbf{IN} and \textbf{RevIN}, but does not constrain the bias.
\end{tcolorbox}

\vspace{1.5mm} 

\textit{Integrating the insights from Lemma~\ref{lemma:dlinear} and Theorem~\ref{theorem:fits} with the analyses presented in this subsection, we establish the following equivalences among the model classes:}

\vspace{0.5mm}
\begin{tcolorbox}[boxsep=3pt,left=5pt,right=5pt,top=5pt,bottom=5pt,colback=blueback,colframe=blueout]
\small
\centering 
\textbf{\boldmath$M(\text{DLinear+IN}) = M(\text{Linear+IN}) = M(\text{FITS+IN}) = M(\text{RLinear})  \approx M(\text{NLinear})$ }
\end{tcolorbox}

\section{Discussion}\label{sec:discussion}


\begin{table}[ht]
\centering
\scriptsize
\setlength{\tabcolsep}{3pt}
\renewcommand{\arraystretch}{1.3} 
\begin{tabular}{|c|c|c|c|}
\hline
\textbf{Model Class} & \textbf{Variants} & \textbf{Normalisation} & \textbf{Constraints} \\ \hline
\multirow{1}{*}{$A\vec{x} + \vec{b}$}     & Linear, DLinear, FITS & None & None \\ \hline
    $\widetilde{A}\vec{x} + \vec{b}$ & NLinear & NowNorm & Rows sum to one \\ \hline
\multirow{4}{*}{$\widetilde{A}\vec{x} + \vec{b}\sigma(\vec{x})$} & \multirow{2}{*}{RLinear} & \multirow{2}{*}{RevIn} & \\ 
 &  &  & Rows sum to one \\ \cline{2-3}
& \multirow{2}{*}{FITS+IN} & \multirow{2}{*}{Instance} & Bias coupled with $\sigma$ \\ 
&  &  &  \\ \hline
\end{tabular}
\caption{A summary of the model classes for the DLinear, FITS, RLinear, NLinear and Linear models. Here $\widetilde{A}$ denotes a matrix whose rows must each sum to one and $\sigma(\vec{x})$ denotes the standard deviation of the components of the context vector.}
\label{table:model_summary}
\end{table}

Our analysis is summarised in Table~\ref{table:model_summary}. When $L\geq T-2$ FITS and DLinear are functionally equivalent to unconstrained linear regression (Definition~\ref{def:model_class}). In Section~\ref{sec:invertible_norms} we looked at the model classes for linear models which use one of the standard normalisation procedures for time series analysis. We saw how using normalisation slightly alters the model class. For example, NLinear is equivalent to restricted linear regression wherein the rows of the weight matrix must sum to 1. We showed that Linear+IN, Linear+RevIN (RLinear \cite{li2023revisiting}), and FITS+IN (i.e., the setting in \citep{xu2023fits}) are equivalent to each other, and differ from NLinear in that the bias is parameterised as $\vec{b}\sigma(\vec{x})$. Perhaps most importantly, each model class can be reformulated as unconstrained linear regression on an augmented feature set, and are solvable in closed form owing to convexity.

\paragraph{Convexity} Each of the models we have discussed train using a mean-squared error (MSE) loss function \citep{xu2023fits, zeng2023transformers}. Linear regression with a mean-squared loss function is a convex optimisation problem. By this we mean that the training loss is a convex function of the parameters. A consequence of convexity is that there exists a unique global optima which minimises the training loss (uniqueness requires that the training data is full rank). Significantly, this means that given the same training data, these models should converge to the same solution, via a suitable optimisation procedure.

\paragraph{Closed Form} An important property of least-squares linear regression is that it admits a closed-form solution. A recap of how one computes a closed-form solution for linear regression and closed form solutions for the three model classes in Table~\ref{table:model_summary} may be found in Appendix \ref{sec:closed_form}. In Section~\ref{sec:experiments} we refer to the closed form solutions as the ordinary least-squares (OLS) models, and we will determine how each model fairs against this closed form approach.

\textbf{Remark:} FITS has two separate training modes. In mode 1 the model is trained by mean-squared error (MSE) between the forecast and the target. In mode 2 an additional term is added to the loss which is the MSE between the context vector and FITS reconstruction. Empirically both settings have similar performance \cite{xu2023fits}. In our analysis and experiments we consider mode 1 only.

\begin{tcolorbox}[boxsep=3pt,left=0pt,right=0pt,top=0pt,bottom=0pt,halign=center,colback=blueback,colframe=blueout]
\small
For each model class in Table~\ref{table:model_summary}, the least-squares optima may be found in \textbf{closed form}.
\end{tcolorbox}

We have hypothesised that the convexity of least-squares linear regression means that each model should converge to the same solution given the same data. Nevertheless, given that each model architecture yields a different parameterisation and initialisation, this still leaves open the possibility that early stopping may impact generalisation. Next, we explore how the parameterisation of FITS has the effect of inducing a much lower learning rate on the bias term compared to that of the weight matrix.

\subsection{The FITS Bias-Term}\label{sec:fitsBias-main}
In Theorem~\ref{theorem:fits} we showed that any FITS model can be written in the form $A\vec{x}+\vec{b}$. Moreover, we showed how one may obtain $A, \vec{b}$ given the weight matrix and the bias of the complex linear layer (Appendix~\ref{lemma:prescription}). Specifically, if $\vec{c}$ denotes the complex bias of the complex linear layer then $\vec{b}= iRFT(\vec{c})$ where iRFT denotes the inverse discrete Fourier transform (Definition~\ref{def:rfft}).  It is important to consider what the implications are of parameterising the bias term in this way, rather than simply parameterising $\vec{b}$ directly. 

\begin{figure*}[!htbp]
	\centering
	\includegraphics[width=0.9\linewidth]{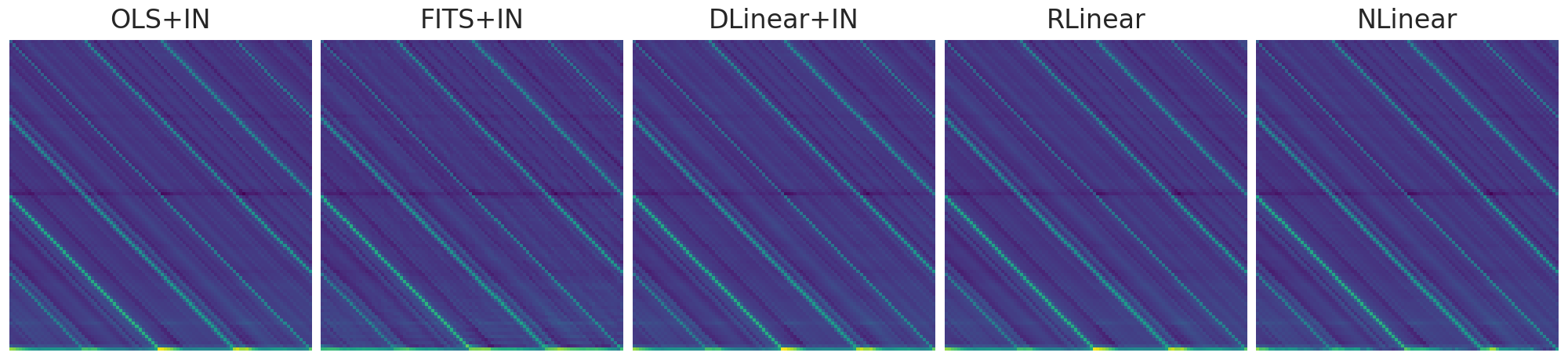}
	\caption{This figure displays the cropped weight matrices after 50 epochs of training for all four models with instance normalization, juxtaposed with their corresponding closed-form solution (extreme left). These show how similar the underlying models are. There are slight differences that affect forecasts to a marginal degree (see Figure \ref{fig:forecasts}).\label{fig:heatmaps}
 }
\end{figure*}

Consider representing the complex vector $\vec{c}$, of dimension $\frac{L+T}{2}$, as a ${(T+L)}$-dimensional real vector. This is achieved by separating the real and imaginary components of $\vec{c}$. Given that the iRFT is a linear mapping, it follows that $\vec{b}$ and $\vec{c}$ are interconnected through the equation $\vec{b} = M\vec{c}$, where $M$ is a specific matrix derived from the iRFT. Critically, the matrix $M$ plays a pivotal role in determining the effective learning rate for the bias in our linear model. For instance, small values within $M$ imply that adjustments in $\Vec{c}$ induce only minor changes in $\vec{b}$. Due the choice of normalisation used for the RFT, the entries of the matrix $M$ are of the order $\sim\frac{1}{\sqrt{L}}$ and FITS manifests this exact phenomena. A detailed breakdown of this may be found in Appendix~\ref{sec:fits-bias}.

\section{Experiments}\label{sec:experiments}
In Section \ref{exp:convergence} we demonstrate that the models discussed in this paper tend toward their corresponding closed form solutions. In Section \ref{exp:performance} we test and compare each model across 8 benchmarking datasets, and show how the \textbf{closed form solution is usually superior}.

\subsection{Convergence}\label{exp:convergence}


\paragraph{Comparison of Learned Weight Matrices} Figure~\ref{fig:heatmaps} visualises the internal weight matrices for 4 trained linear model variants plus the closed-form solution (denoted \textbf{OLS+IN}). The models shown are RLinear, NLinear, DLinear+IN, FITS+IN (the SoTA variant of FITS from \cite{xu2023fits}), and OLS+IN. Each model is trained for 50 epochs\footnote{except for OLS which is solved using an SVD solver in Scikit-learn \cite{sklearn}.} on the ETTh1 dataset with a context of 720 and a prediction horizon length of 336. The weight matrices are then extracted and visualised using the same colour scale. In all cases the learned matrices are near identical. The similarity of the weight matrices for Linear+IN, RLinear, FITS+IN and OLS+IN is precisely in line with our hypothesis and matches the theory and discussion from previous sections. Note that while NLinear lies is a slightly different model class (See Table~\ref{table:model_summary}), the learned matrix is still near identical.

Plotting the learned matrices as in Figure~\ref{fig:heatmaps} requires us to first convert each trained model into the form $f(\vec{x})=A\vec{x}+\vec{b}$. To do this we note that $f(\vec{0}) = A\vec{0}+\vec{b}=\vec{b}$. Thus, the bias can be found by passing the zero vector into the trained model. We can determine $A$ in a similar manner. Let $\vec{e_i}$ denote the $i$\textsuperscript{th} coordinate vector, that is $\vec{e_i}$ is the vector which is 1 at position $i$ and zero elsewhere. Then $f(\vec{e_i})=A\vec{e_i}+\vec{b} = A_{\cdot, i} + \vec{b}$ where $A_{\cdot, i}$ is the $i$\textsuperscript{th} column of $A$. Hence, given that we have already computed the bias term, we may derive $A$ simply by passing through each coordinate vector $\vec{e_i}$ and subtracting $\vec{b}$. The procedure for extracting the weight matrices for models of the form $A\vec{x}+\vec{b}\sigma(x)$ is similar and is discussed in Appendix \ref{sec:extracting_matrices}. 
\begin{figure}[!htbp]
	\centering
	\includegraphics[width=0.95\linewidth]{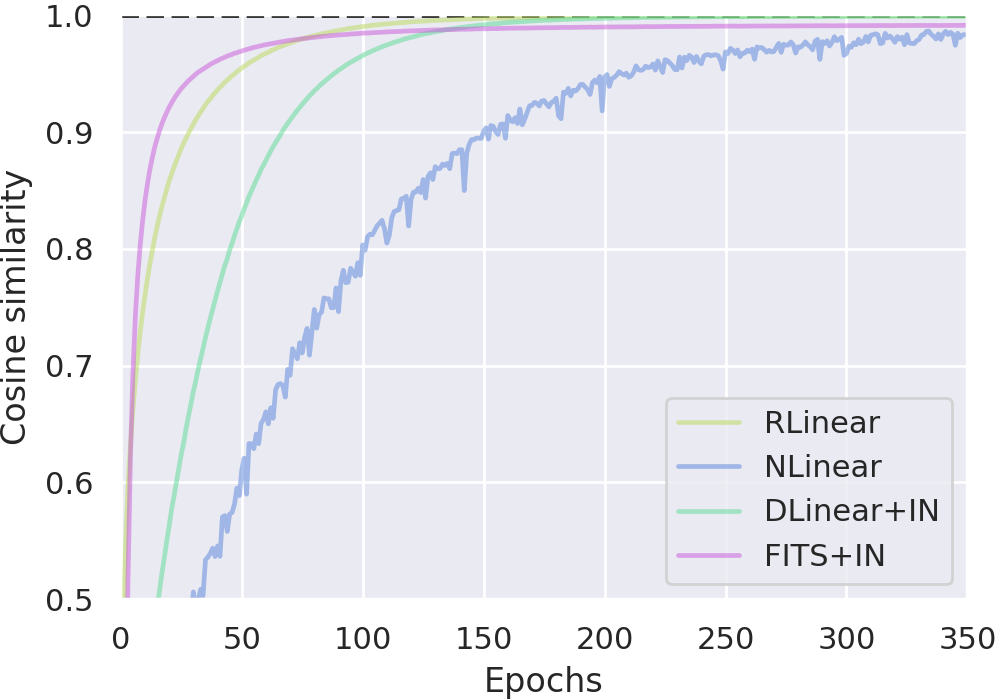}
	\caption{A demonstration of how the model's weight matrices tend to the OLS solution during training. This is a visualised as the cosine similarity between a given model's weight matrix and that determined by the closed form solution. \vspace{-0.5cm}
 \label{fig:cosines}
 }
\end{figure}
\paragraph{Cosine similarity during training} Figure \ref{fig:cosines} tracks the cosine similarity (Defined $d(x,y):=\frac{x\cdot y}{\vert\vert x\vert\vert_2\cdot \vert\vert y\vert\vert_2}$) between the above-mentioned 4 models' weight matrices and their OLS counterpart during training. A cosine similarity of one, corresponds to exact equality between the matrices.  In line with our hypothesis, all model's weight converge toward the OLS solution. The rapidity of this behaviour differs per model, thus demonstrating that SGD optimisation coupled with each unique parameterisation impacts the particularly route taken and rate of convergence. 

\paragraph{Forecasts} Figure~\ref{fig:forecasts} shows the forecasts from these models after 50 epochs of training.  While subtle differences in the models do indeed result in subtle differences in forecasts, there is clear and pervasive similarity between forecasts. 
\begin{figure}[!htbp]
	\centering
	\includegraphics[width=0.99\linewidth]{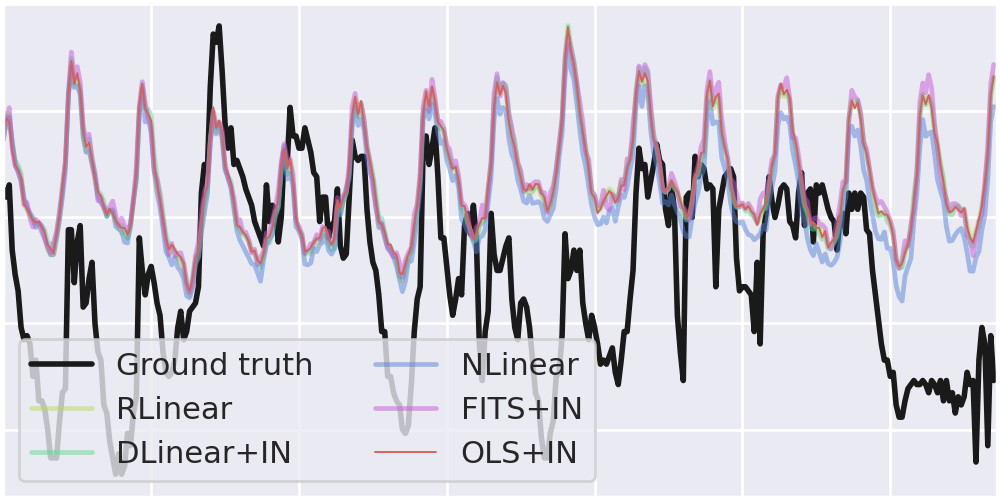}
	\caption{Forecast comparison on ETTh1 with $T=336$, comparing the 5 models that use instance normalisation. \vspace{-0.5cm}
 \label{fig:forecasts}
 }
\end{figure}

\paragraph{Bias Terms} The bias terms for each trained model are visualised in Figure~\ref{fig:biases}. As expected, the models DLinear+IN, RLinear and OLS+IN learn the same bias terms as each other. Notably however, the bias for FITS+IN differs considerably from the other models. Moreover the magnitude of this bias is much smaller. This difference is despite the fact that all these models' classes are equivalent (Table~\ref{table:model_summary}). This confirms our analysis from Section \ref{sec:fitsBias-main}.

\begin{figure}[!htbp]
	\centering
	\includegraphics[width=0.99\linewidth]{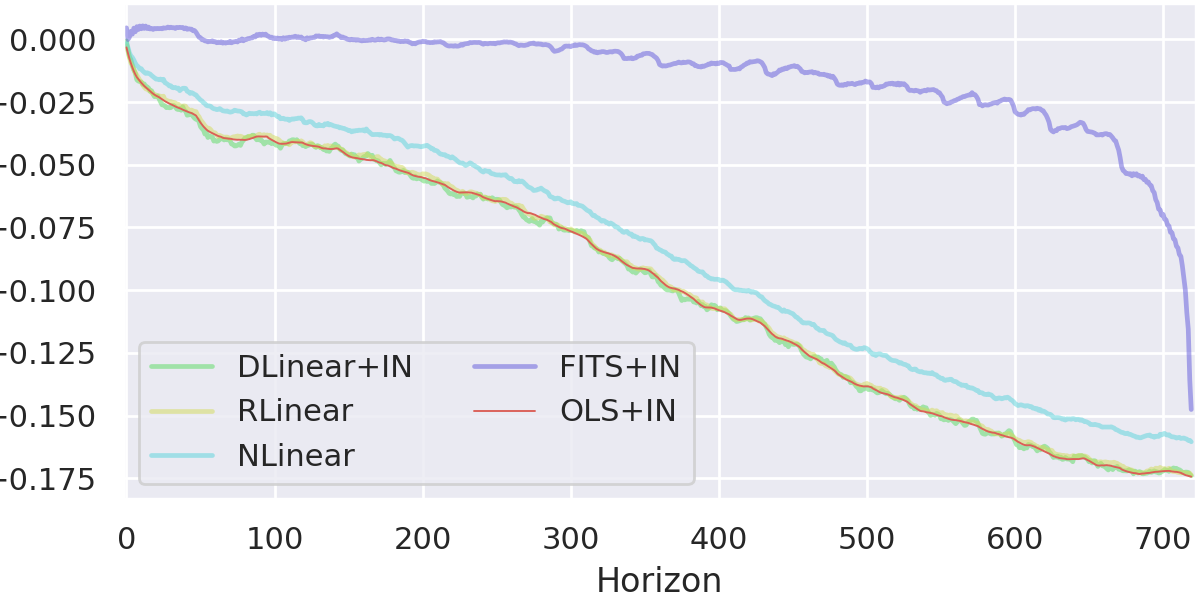}
	\caption{Comparison of the learned bias parameters for several linear models implementing feature normalisation technique. FITS results clearly in a different bias term.\vspace{-0.5cm}
 \label{fig:biases}
 }
\end{figure}

\subsection{Performance}\label{exp:performance}

\begin{table*}[!ht]
\small
\caption{Long-term multivariate forecasting results, showing MSE values for all models investigated in this work. The \textbf{\color{green2}green} and \textbf{\color{blue2}{blue}} highlighting indicate when the OLS is superior and within 1 standard deviation of a given model, respectively. \textbf{Bold}ing indicates the best performing model for a given dataset-horizon combination.}
\label{tab:comparison}
\begin{center}
\setlength{\tabcolsep}{3pt}
\begin{tabular}{| c | c | c | c | c | c V{4} c | c | c | c | c | }
\hline
& & \multicolumn{4}{c V{4}}{Methods \textbf{without} instance normalisation} & \multicolumn{5}{c|}{Methods \textbf{with} instance normalisation}\\

& $T$ & \multicolumn{1}{c}{\textbf{OLS}} & \multicolumn{1}{c}{FITS} & \multicolumn{1}{c}{DLinear} & \multicolumn{1}{c V{4}}{Linear} & \multicolumn{1}{c}{\textbf{OLS+IN}} & \multicolumn{1}{c}{FITS+IN} & \multicolumn{1}{c}{DLinear+IN} & \multicolumn{1}{c}{RLinear} & \multicolumn{1}{c|}{NLinear}  \\
 \hline 
\multirow{4}{0.5em}{\rotatebox{90}{ETTm1}} 
& 96  & \textbf{0.306} & \cellcolor{green}  0.310 $_{\pm 0.0005}$ & \cellcolor{green}  0.311 $_{\pm 0.0008}$ & \cellcolor{green}  0.314 $_{\pm 0.0037}$ & 0.307  & \cellcolor{green} 0.309 $_{\pm 0.0002}$   & \cellcolor{green}   0.312 $_{\pm 0.0008}$ & \cellcolor{green}   0.312 $_{\pm 0.0024}$ & \cellcolor{green}   0.319 $_{\pm 0.0021}$  \\ 
& 192 & \textbf{0.335} & \cellcolor{green}  0.338 $_{\pm 0.0008}$ & \cellcolor{green}  0.342 $_{\pm 0.0014}$ & \cellcolor{green}  0.343 $_{\pm 0.0012}$ & 0.336  & \cellcolor{green} 0.338 $_{\pm 0.0005}$  & \cellcolor{green}   0.341 $_{\pm 0.0014}$ & \cellcolor{green}   0.343 $_{\pm 0.0010}$ & \cellcolor{green}   0.346 $_{\pm 0.0009}$  \\ 
& 336 & \textbf{0.364} & \cellcolor{green}  0.367 $_{\pm 0.0003}$ & \cellcolor{green}  0.372 $_{\pm 0.0006}$ & \cellcolor{green}  0.374 $_{\pm 0.0008}$ & 0.365  & \cellcolor{green} 0.367 $_{\pm 0.0001}$  & \cellcolor{green}   0.372 $_{\pm 0.0006}$ & \cellcolor{green}   0.372 $_{\pm 0.0016}$ & \cellcolor{green}   0.378 $_{\pm 0.0003}$ \\ 
& 720 & \textbf{0.413} & \cellcolor{green}  0.435 $_{\pm 0.0010}$ & \cellcolor{green}  0.422 $_{\pm 0.0016}$ & \cellcolor{green}  0.426 $_{\pm 0.0058}$ & 0.415  & \cellcolor{green} 0.417 $_{\pm 0.0006}$  & \cellcolor{green}  0.422 $_{\pm 0.0016}$ & \cellcolor{green}   0.421 $_{\pm 0.0018}$ & \cellcolor{green}  0.424 $_{\pm 0.0029}$  \\ 
\hline 
\multirow{4}{0.5em}{\rotatebox{90}{ETTm\textbf{}2}} 
& 96  & 0.166 & \cellcolor{white}  0.165 $_{\pm 0.0003}$ & \cellcolor{white}  0.164 $_{\pm 0.0017}$ & \cellcolor{white}  0.163 $_{\pm 0.0010}$ & \textbf{0.162} & \cellcolor{blue} \textbf{0.162} $_{\pm 0.0001}$ & \cellcolor{blue} 0.163 $_{\pm 0.0011}$ & \cellcolor{green}   0.164 $_{\pm 0.0009}$ & \cellcolor{green}   0.164 $_{\pm 0.0009}$ \\ 
& 192 & 0.228 & \cellcolor{white}  0.225 $_{\pm 0.0001}$ &  \cellcolor{white} 0.222 $_{\pm 0.0023}$ & \cellcolor{white}  0.218 $_{\pm 0.0013}$ & \textbf{0.216} & \cellcolor{green}  0.217 $_{\pm 0.0001}$ & \cellcolor{green}   0.217 $_{\pm 0.0004}$ & \cellcolor{green}   0.217 $_{\pm 0.0007}$ & \cellcolor{green}   0.217 $_{\pm 0.0007}$ \\ 
& 336 & 0.295 & \cellcolor{white}  0.291 $_{\pm 0.0008}$ & \cellcolor{white} \textbf{0.267} $_{\pm 0.0029}$ & \cellcolor{white}  0.272 $_{\pm 0.0021}$ & 0.268 & \cellcolor{green}  0.269 $_{\pm 0.0000}$ & \cellcolor{green}  0.269 $_{\pm 0.0007}$ & \cellcolor{green}  0.270 $_{\pm 0.0011}$ & \cellcolor{green}   0.270 $_{\pm 0.0011}$ \\ 
& 720 & 0.415 & \cellcolor{white}  0.409 $_{\pm 0.0004}$ & \cellcolor{white}  0.356 $_{\pm 0.0056}$ & \cellcolor{white}  0.362 $_{\pm 0.0073}$ & \textbf{0.349} & \cellcolor{green}  0.350 $_{\pm 0.0002}$ & \cellcolor{green}  0.354 $_{\pm 0.0016}$ & \cellcolor{green}   0.354 $_{\pm 0.0010}$ & \cellcolor{green}  0.355 $_{\pm 0.0010}$ \\ 
\hline 
\multirow{4}{0.5em}{\rotatebox{90}{ETTh1}} 
& 96  & 0.376 & \cellcolor{green}  0.378 $_{\pm 0.0002}$ & \cellcolor{green} 0.380 $_{\pm 0.0027}$  & \cellcolor{green}  0.390 $_{\pm 0.0016}$ & \textbf{0.375} & \cellcolor{green}  0.377 $_{\pm 0.0002}$ &\cellcolor{green}   0.379 $_{\pm 0.0010}$ & \cellcolor{green}   0.387 $_{\pm 0.0006}$ & \cellcolor{green}   0.383 $_{\pm 0.0027}$\\ 
& 192 & \textbf{0.413} & \cellcolor{blue} \textbf{0.413} $_{\pm 0.0002}$ & \cellcolor{green}  0.424 $_{\pm 0.0045}$  & \cellcolor{green}  0.426 $_{\pm 0.0029}$ & \textbf{0.413} & \cellcolor{blue} \textbf{0.413} $_{\pm 0.0002}$ & \cellcolor{green}  0.419 $_{\pm 0.0018}$ &\cellcolor{green}   0.415 $_{\pm 0.0015}$ & \cellcolor{green}  0.418 $_{\pm 0.0016}$ \\ 
& 336 & 0.448 & \cellcolor{green}  0.500 $_{\pm 0.0014}$ & \cellcolor{green}   0.458$_{\pm 0.0104}$ & \cellcolor{green}  0.465 $_{\pm 0.0044}$ & 0.445 & \cellcolor{white}  \textbf{0.432} $_{\pm 0.0008}$ & \cellcolor{green}  0.451 $_{\pm 0.0020}$ &\cellcolor{green}   0.450 $_{\pm 0.0007}$ &\cellcolor{green}   0.446 $_{\pm 0.0006}$\\ 
& 720 & 0.491 & \cellcolor{green}  0.506 $_{\pm 0.0062}$ & \cellcolor{green}  0.522 $_{\pm 0.0051}$ & \cellcolor{green}  0.512 $_{\pm 0.0017}$ & 0.460 & \cellcolor{white}  \textbf{0.428} $_{\pm 0.0002}$ &\cellcolor{green}   0.470 $_{\pm 0.0013}$ &\cellcolor{blue}  0.460 $_{\pm 0.0074}$ & \cellcolor{green}   0.464 $_{\pm 0.0006}$\\ 
\hline 
\multirow{4}{0.5em}{\rotatebox{90}{ETTh2}} 
& 96  & 0.309 & \cellcolor{white}   0.307 $_{\pm 0.0008}$ & \cellcolor{white}   \textbf{0.276} $_{\pm 0.0013}$ & \cellcolor{white}  0.277 $_{\pm 0.0119}$ & 0.270 & \cellcolor{blue} 0.270 $_{\pm 0.0001}$ & \cellcolor{green}  0.275 $_{\pm 0.0002}$ & \cellcolor{green}  0.272 $_{\pm 0.0015}$ & \cellcolor{green}  0.279 $_{\pm 0.0020}$ \\ 
& 192 & 0.423 & \cellcolor{green}   0.447 $_{\pm 0.0019}$ & \cellcolor{white}  
 0.351 $_{\pm 0.0140}$ & \cellcolor{white}  0.351 $_{\pm 0.0123}$ & \textbf{0.331} & \cellcolor{blue} \textbf{0.331} $_{\pm 0.0000}$ & \cellcolor{green}  0.342 $_{\pm 0.0025}$ &\cellcolor{green}   0.335 $_{\pm 0.0010}$ &\cellcolor{green}   0.343 $_{\pm 0.0026}$  \\ 
& 336 & 0.540 & \cellcolor{green}   0.566 $_{\pm 0.0014}$ & \cellcolor{white}   0.424 $_{\pm 0.0139}$ & \cellcolor{white}  0.455 $_{\pm 0.0053}$ & \textbf{0.353} & \cellcolor{green}  0.354 $_{\pm 0.0001}$ &\cellcolor{green}   0.359 $_{\pm 0.0062}$ & \cellcolor{green}  0.357 $_{\pm 0.0011}$ &\cellcolor{green}   0.383 $_{\pm 0.0028}$ \\ 
& 720 & 0.900 & \cellcolor{green}  0.971 $_{\pm 0.0018}$ & \cellcolor{white}  
 0.664 $_{\pm 0.0400}$ & \cellcolor{white}  0.619 $_{\pm 0.0185}$ & 0.380 & \cellcolor{white}  \textbf{0.377} $_{\pm 0.0001}$ & \cellcolor{green}  0.384 $_{\pm 0.0001}$ &\cellcolor{green}   0.384 $_{\pm 0.0012}$ & \cellcolor{green}  0.406 $_{\pm 0.0054}$ \\ 
\hline 
\multirow{4}{0.5em}{\rotatebox{90}{ECL}} 
& 96  & \textbf{0.133} & \cellcolor{green}  0.134 $_{\pm 0.0002}$ & \cellcolor{green}  
 0.134 $_{\pm 0.0001}$ & \cellcolor{blue} \textbf{0.133} $_{\pm 0.0002}$ & \textbf{0.133} & \cellcolor{blue} \textbf{0.133} $_{\pm 0.0001}$ & \cellcolor{green}   0.134 $_{\pm 0.0001}$ & \cellcolor{green}   0.134 $_{\pm 0.0001 }$ & \cellcolor{green}  0.134 $_{\pm 0.0002}$ \\ 
& 192 & \textbf{0.147} & \cellcolor{green}  0.148 $_{\pm 0.0001}$ & \cellcolor{green}  
 0.148 $_{\pm 0.0005}$ & \cellcolor{green}   0.148 $_{\pm 0.0001}$ & 0.148 & \cellcolor{blue} 0.148 $_{\pm 0.0000}$ & \cellcolor{green}   0.149 $_{\pm 0.0000}$ & \cellcolor{blue}  0.148 $_{\pm 0.0000}$ & \cellcolor{blue} 0.149 $_{\pm 0.0002}$ \\ 
& 336 & \textbf{0.162} & \cellcolor{green}  0.164 $_{\pm 0.0002}$ & \cellcolor{green}  
 0.164 $_{\pm 0.0009}$ &\cellcolor{green}   0.163 $_{\pm 0.0001}$ & 0.164 & \cellcolor{blue} 0.164 $_{\pm 0.0001}$ & \cellcolor{green}  0.165 $_{\pm 0.0001}$ & \cellcolor{green}  0.165 $_{\pm 0.0001}$ & \cellcolor{blue} 0.165 $_{\pm 0.0001}$ \\ 
& 720 & \textbf{0.197} & \cellcolor{green}  0.200 $_{\pm 0.0001}$ & \cellcolor{blue} \textbf{0.197} $_{\pm 0.0034}$ & \cellcolor{green}  0.198 $_{\pm 0.0002}$ & 0.203 & \cellcolor{blue} 0.203 $_{\pm 0.0000}$ & \cellcolor{green}  0.205 $_{\pm 0.0003}$ & \cellcolor{green}  0.204 $_{\pm 0.0001}$ & \cellcolor{green}  0.205 $_{\pm 0.0002}$ \\ 
\hline 
\multirow{4}{0.5em}{\rotatebox{90}{Traffic}} 
& 96  & \textbf{0.385} & \cellcolor{green}  0.386 $_{\pm 0.0003}$ &  \cellcolor{green}  
 0.387 $_{\pm 0.0003}$ & \cellcolor{green} 0.386 $_{\pm 0.0005}$ & \textbf{0.385} & \cellcolor{green}  0.386 $_{\pm 0.0002}$ & \cellcolor{green}  0.387 $_{\pm 0.0002}$ & \cellcolor{green} 0.386 $_{\pm 0.0004}$ & \cellcolor{green} 0.387 $_{\pm 0.0003}$ \\ 
& 192 & \textbf{0.396} & \cellcolor{green}  0.397 $_{\pm 0.0001}$ &  \cellcolor{green}   0.398 $_{\pm 0.0001}$ & \cellcolor{green} 0.398 $_{\pm 0.0003}$ & 0.397 & \cellcolor{green}  0.398 $_{\pm 0.0001}$ & \cellcolor{green}  0.399 $_{\pm 0.0003}$ & \cellcolor{green} 0.397 $_{\pm 0.0004}$ & \cellcolor{green} 0.398 $_{\pm 0.0000}$ \\ 
& 336 & \textbf{0.410} & \cellcolor{green}  0.411 $_{\pm 0.0001}$ & \cellcolor{green}  0.412 $_{\pm 0.0001}$ & \cellcolor{green} 0.412 $_{\pm 0.0002}$ & \textbf{0.410} & \cellcolor{green}  0.411 $_{\pm 0.0001}$ & \cellcolor{green}  0.412 $_{\pm 0.0005}$ & \cellcolor{green} 0.412 $_{\pm 0.0000}$ & \cellcolor{green} 0.412 $_{\pm 0.0000}$ \\ 
& 720 & 0.450 & \cellcolor{blue} 0.450 $_{\pm 0.0002}$ & \cellcolor{blue} 0.450 $_{\pm 0.0006}$ & \cellcolor{green} 0.451 $_{\pm 0.0003}$ & \textbf{0.448} & \cellcolor{green}  0.449 $_{\pm 0.0001}$ & \cellcolor{green}  0.450 $_{\pm 0.0002}$ & 0.449 $_{\pm 0.0002}$ &  \cellcolor{green} \cellcolor{green} 0.451 $_{\pm 0.0000}$ \\ 
\hline 
\multirow{4}{0.5em}{\rotatebox{90}{Weather}} 
& 96  & 0.142 & \cellcolor{green}  0.144 $_{\pm 0.0002}$ & \cellcolor{green}  0.145 $_{\pm 0.0017}$ & \cellcolor{green}  0.145 $_{\pm 0.0011}$ & \textbf{0.141} &\cellcolor{green}  0.142 $_{\pm 0.0000}$ & \cellcolor{green} 0.142 $_{\pm 0.0006}$ &\cellcolor{green}  0.143 $_{\pm 0.0005}$ & \cellcolor{green} 0.144 $_{\pm 0.0004}$  \\ 
& 192 & 0.185 & \cellcolor{green}  0.188 $_{\pm 0.0013}$ & \cellcolor{green}  0.188 $_{\pm 0.0029}$ & \cellcolor{green}  0.189 $_{\pm 0.0028}$ & \textbf{0.184} & \cellcolor{green} 0.185 $_{\pm 0.0001}$ & \cellcolor{green} 0.185 $_{\pm 0.0008}$ & \cellcolor{green} 0.185 $_{\pm 0.0007}$ & \cellcolor{green} 0.187 $_{\pm 0.0010}$  \\ 
& 336 & 0.235 & \cellcolor{green}  0.238 $_{\pm 0.0003}$ & \cellcolor{blue} 0.235 $_{\pm 0.0004}$ & \cellcolor{green}  0.238 $_{\pm 0.0019}$ & \textbf{0.234} &\cellcolor{green}  0.236 $_{\pm 0.0001}$ & \cellcolor{green} 0.235 $_{\pm 0.0003}$ & \cellcolor{green} 0.235 $_{\pm 0.0005}$ & \cellcolor{green} 0.235 $_{\pm 0.0002}$  \\ 
& 720 & \textbf{0.304} & \cellcolor{blue} \textbf{0.304} $_{\pm 0.0004}$ & \cellcolor{green}  
 0.308 $_{\pm 0.0005}$ & \cellcolor{green}  0.310 $_{\pm 0.0018}$ & 0.307 &\cellcolor{blue} 0.307 $_{\pm 0.0001}$ & \cellcolor{green} 0.310 $_{\pm 0.0004}$ & \cellcolor{green} 0.309 $_{\pm 0.0006}$ & \cellcolor{green} 0.311 $_{\pm 0.0003}$  \\ 
\hline 
\multirow{4}{0.5em}{\rotatebox{90}{Exchange}} 
& 96  & 0.091 & \cellcolor{green}  0.099 $_{\pm 0.0009}$ & \cellcolor{white}  \textbf{0.084} $_{\pm 0.0003}$ & \cellcolor{green}  0.100 $_{\pm 0.0097}$ & 0.086 & \cellcolor{green}  0.087 $_{\pm 0.0001}$ & \cellcolor{white}   0.085 $_{\pm 0.0003}$ &\cellcolor{blue} 0.086 $_{\pm 0.0006}$ & \cellcolor{green}    0.090 $_{\pm 0.0008}$  \\ 
& 192 & 0.217 & \cellcolor{green}  0.243 $_{\pm 0.0032}$ & \cellcolor{white}  \textbf{0.160} $_{\pm 0.0088}$ & \cellcolor{white}  0.161 $_{\pm  0.0012}$ & 0.180 & \cellcolor{green}  0.183 $_{\pm 0.0005}$ & \cellcolor{blue}   0.178 $_{\pm 0.0028}$ & \cellcolor{blue}  0.179 $_{\pm 0.0024}$ & \cellcolor{green}   0.187 $_{\pm 0.0034}$  \\ 
& 336 & 0.450 & \cellcolor{green}  0.498 $_{\pm 0.0026}$ & \cellcolor{white}  \textbf{0.315} $_{\pm 0.0070}$ & \cellcolor{white}  0.323 $_{\pm 0.0126}$ & 0.343 & \cellcolor{green}  0.344 $_{\pm 0.0011}$ & \cellcolor{white}  0.335 $_{\pm 0.0031}$ & \cellcolor{white}  0.334 $_{\pm 0.0040}$ & \cellcolor{green}   0.347 $_{\pm 0.0024}$  \\ 
& 720 & 1.392 & \cellcolor{white}  1.256 $_{\pm 0.0083}$ & \cellcolor{white}  0.929 $_{\pm 0.0218}$ & \cellcolor{white}  \textbf{0.717} $_{\pm 0.1699}$ & 0.992 & \cellcolor{white}  0.965 $_{\pm 0.0010}$ & \cellcolor{white}  0.920 $_{\pm 0.0219}$ & \cellcolor{white}  0.948 $_{\pm 0.0082}$ & \cellcolor{green}   1.035 $_{\pm 0.0130}$  \\ 
\hline 
\end{tabular}
\end{center}
\end{table*}

Table \ref{tab:comparison} presents the Mean Squared Error (MSE) values, accompanied by error bars, for the models evaluated in this study, both with and without instance normalization.\footnote{We included NLinear in the grouping `with' instance normalisation, even though the model classes is slightly different.} In the table, \textbf{\color{green2}green} highlighting signifies instances where the Ordinary Least Squares (OLS) solution achieves a lower MSE compared to the model being evaluated. Conversely, \textbf{\color{blue2}{blue}} highlighting denotes cases where the differences are within one standard deviation.


Table~\ref{tab:comparison} shows that the linear models are generally outperformed by their corresponding OLS solution ($72
\%$ of settings). It is interesting that the OLS solution usually outperforms those trained with SGD and early stopping, particularly given that the OLS solutions are purely linear regression (not ridge or lasso regression), meaning that there is no regularisation. The comparably strong performance of the closed-form solution on larger datasets (ECL, Traffic, and Weather) suggests that a linear model may not have sufficient representational capacity in this setting. 

Conversely, FITS performs particularly well on the hourly ETT dataset (ETTh1 and h2). We believe that the reason for this is owed to the fact that these datasets are small, such that overfitting can occur rapidly. Since FITS inadvertently imposes a restriction on the bias parameter (see Section \ref{sec:fitsBias-main} and Figure \ref{fig:biases}), it is less prone to this overfitting restriction. 

\begin{tcolorbox}[boxsep=3pt,left=0pt,right=0pt,top=0pt,bottom=0pt,halign=center,colback=blueback,colframe=blueout]
\small
OLS solutions were superior across 23 of 32 (72\%) settings.
\end{tcolorbox}
\section{Conclusion}
Simple linear models are often on par, or better, than complex or deep models for time series forecasting. Thus, much energy has thus been spent on `modernising' linear regression for time series forecasting: modelling separately trends and residuals (DLinear), applying some form of instance normalisation (RLinear, NLinear), or by processing in Fourier space (FITS). We have shown in this paper that, from a functional standpoint, these alterations barely deviate these models from standard linear regression. We demonstrated empirically that these model behave and perform similarly to each other and generally worse than their closed-form solutions. A full discussion of the limitations and future of this work may be found in Appendix \ref{sec:limitations}.





\bibliography{example_paper}

\begin{thebibliography}{21}
\providecommand{\natexlab}[1]{#1}
\providecommand{\url}[1]{\texttt{#1}}
\expandafter\ifx\csname urlstyle\endcsname\relax
  \providecommand{\doi}[1]{doi: #1}\else
  \providecommand{\doi}{doi: \begingroup \urlstyle{rm}\Url}\fi

\bibitem[Anonymous(2024)]{anonymous2024dam}
Anonymous.
\newblock {DAM}: A foundation model for forecasting.
\newblock In \emph{The Twelfth International Conference on Learning Representations}, 2024.
\newblock URL \url{https://openreview.net/forum?id=4NhMhElWqP}.

\bibitem[Darlow et~al.(2023)Darlow, Joosen, Asenov, Deng, Wang, and Barker]{darlow2023foldformer}
Darlow, L.~N., Joosen, A., Asenov, M., Deng, Q., Wang, J., and Barker, A.
\newblock {FoldFormer}: Sequence folding and seasonal attention for fine-grained long-term {FaaS} forecasting.
\newblock In \emph{Proceedings of the 3rd Workshop on Machine Learning and Systems}, pp.\  71--77, 2023.

\bibitem[Hastie et~al.(2009)Hastie, Tibshirani, Friedman, and Friedman]{hastie2009elements}
Hastie, T., Tibshirani, R., Friedman, J.~H., and Friedman, J.~H.
\newblock \emph{The elements of statistical learning: data mining, inference, and prediction}, volume~2.
\newblock Springer, 2009.

\bibitem[Joosen et~al.(2023)Joosen, Hassan, Asenov, Singh, Darlow, Wang, and Barker]{Joosen2023How}
Joosen, A., Hassan, A., Asenov, M., Singh, R., Darlow, L., Wang, J., and Barker, A.
\newblock How does it function? characterizing long-term trends in production serverless workloads.
\newblock In \emph{Proceedings of the 2023 ACM Symposium on Cloud Computing}, SoCC '23, pp.\  443–458, New York, NY, USA, 2023. Association for Computing Machinery.
\newblock ISBN 9798400703874.
\newblock \doi{10.1145/3620678.3624783}.
\newblock URL \url{https://doi.org/10.1145/3620678.3624783}.

\bibitem[Kim et~al.(2021)Kim, Kim, Tae, Park, Choi, and Choo]{kim2021reversible}
Kim, T., Kim, J., Tae, Y., Park, C., Choi, J.-H., and Choo, J.
\newblock Reversible instance normalization for accurate time-series forecasting against distribution shift.
\newblock In \emph{International Conference on Learning Representations}, 2021.

\bibitem[Lai et~al.(2018)Lai, Chang, Yang, and Liu]{lai2018modeling}
Lai, G., Chang, W.-C., Yang, Y., and Liu, H.
\newblock Modeling long-and short-term temporal patterns with deep neural networks.
\newblock In \emph{The 41st international ACM SIGIR conference on research \& development in information retrieval}, pp.\  95--104, 2018.

\bibitem[Li et~al.(2023)Li, Qi, Li, and Xu]{li2023revisiting}
Li, Z., Qi, S., Li, Y., and Xu, Z.
\newblock Revisiting long-term time series forecasting: An investigation on linear mapping, 2023.

\bibitem[Liu et~al.(2021)Liu, Yu, Liao, Li, Lin, Liu, and Dustdar]{liu2021pyraformer}
Liu, S., Yu, H., Liao, C., Li, J., Lin, W., Liu, A.~X., and Dustdar, S.
\newblock Pyraformer: Low-complexity pyramidal attention for long-range time series modeling and forecasting.
\newblock In \emph{International conference on learning representations}, 2021.

\bibitem[Liu et~al.(2023)Liu, Hu, Zhang, Wu, Wang, Ma, and Long]{liu2023itransformer}
Liu, Y., Hu, T., Zhang, H., Wu, H., Wang, S., Ma, L., and Long, M.
\newblock itransformer: Inverted transformers are effective for time series forecasting, 2023.

\bibitem[Nie et~al.(2022)Nie, Nguyen, Sinthong, and Kalagnanam]{nie2022time}
Nie, Y., Nguyen, N.~H., Sinthong, P., and Kalagnanam, J.
\newblock A time series is worth 64 words: Long-term forecasting with transformers.
\newblock In \emph{The Eleventh International Conference on Learning Representations}, 2022.

\bibitem[Pedregosa et~al.(2011)Pedregosa, Varoquaux, Gramfort, Michel, Thirion, Grisel, Blondel, Prettenhofer, Weiss, Dubourg, et~al.]{sklearn}
Pedregosa, F., Varoquaux, G., Gramfort, A., Michel, V., Thirion, B., Grisel, O., Blondel, M., Prettenhofer, P., Weiss, R., Dubourg, V., et~al.
\newblock Scikit-learn: Machine learning in python.
\newblock \emph{Journal of machine learning research}, 12\penalty0 (Oct):\penalty0 2825--2830, 2011.

\bibitem[Silva(2024)]{Silva_SVDRegression}
Silva, T.
\newblock Understanding linear regression using the singular value decomposition, 2024.
\newblock URL \url{https://sthalles.github.io/svd-for-regression/}.
\newblock Online; accessed Day Month Year.

\bibitem[Sloss et~al.(2019)Sloss, Nukala, and Rau]{sloss2019metrics}
Sloss, B.~T., Nukala, S., and Rau, V.
\newblock Metrics that matter.
\newblock \emph{Communications of the ACM}, 62\penalty0 (4):\penalty0 88--88, 2019.

\bibitem[Taylor \& Letham(2018)Taylor and Letham]{taylor2018forecasting}
Taylor, S.~J. and Letham, B.
\newblock Forecasting at scale.
\newblock \emph{The American Statistician}, 72\penalty0 (1):\penalty0 37--45, 2018.

\bibitem[Vaswani et~al.(2017)Vaswani, Shazeer, Parmar, Uszkoreit, Jones, Gomez, Kaiser, and Polosukhin]{vaswani2017attention}
Vaswani, A., Shazeer, N., Parmar, N., Uszkoreit, J., Jones, L., Gomez, A.~N., Kaiser, {\L}., and Polosukhin, I.
\newblock Attention is all you need.
\newblock \emph{Advances in neural information processing systems}, 30, 2017.

\bibitem[Wu et~al.(2021)Wu, Xu, Wang, and Long]{wu2021autoformer}
Wu, H., Xu, J., Wang, J., and Long, M.
\newblock Autoformer: Decomposition transformers with auto-correlation for long-term series forecasting.
\newblock \emph{Advances in Neural Information Processing Systems}, 34:\penalty0 22419--22430, 2021.

\bibitem[Xu et~al.(2023)Xu, Zeng, and Xu]{xu2023fits}
Xu, Z., Zeng, A., and Xu, Q.
\newblock Fits: Modeling time series with $10 k $ parameters.
\newblock \emph{arXiv preprint arXiv:2307.03756}, 2023.

\bibitem[Zeng(2023)]{DlinearGITHUB}
Zeng, A.
\newblock Ltsf-linear.
\newblock \url{https://github.com/cure-lab/LTSF-Linear}, 2023.

\bibitem[Zeng et~al.(2023)Zeng, Chen, Zhang, and Xu]{zeng2023transformers}
Zeng, A., Chen, M., Zhang, L., and Xu, Q.
\newblock Are transformers effective for time series forecasting?
\newblock In \emph{Proceedings of the AAAI conference on artificial intelligence}, volume~37, pp.\  11121--11128, 2023.

\bibitem[Zhijian(2023)]{FITSGITHUB}
Zhijian, X.
\newblock Fits.
\newblock \url{https://github.com/VEWOXIC/FITS}, 2023.

\bibitem[Zhou et~al.(2021)Zhou, Zhang, Peng, Zhang, Li, Xiong, and Zhang]{zhou2021informer}
Zhou, H., Zhang, S., Peng, J., Zhang, S., Li, J., Xiong, H., and Zhang, W.
\newblock Informer: Beyond efficient transformer for long sequence time-series forecasting.
\newblock In \emph{Proceedings of the AAAI conference on artificial intelligence}, volume~35, pp.\  11106--11115, 2021.

\end{thebibliography}
\bibliographystyle{icml2024}

\newpage
\appendix
\onecolumn
\section{Appendix}

\subsection{FITS}\label{sec:fitsProof}
This section is dedicated to fully unpacking the FITS model and proving Theorem~\ref{theorem:fits}.

\textbf{FITS Model Definition:} Let $\vec{x}\in \mathbb{R}^L$ be a context vector. FITS applies the Real (discrete) Fourier Transform (RFT) to $\vec{x}$. This maps $\vec{x}$ to a complex vector of length $\floor{L/2}+1$. After this one applies a Low-Pass Filter (LPF), zeroing out the high frequency components. Next one applies a learnable complex linear layer. The output is padded with zeros and the result is passed though the inverse RFT, mapping to $\mathbb{R}^{L+T}$. The result is then scaled by $\frac{L+T}{L}$.

Throughout this section we will assume both the prediction horizon length $T$ and the context length $L$ are even. This will avoid over-cluttered expressions involving the floor functions. Moreover this condition holds for every experiment setting in the original paper \cite{xu2023fits}.

\textbf{Remark:} FITS is a state-of-art model. One of the goals of this paper is to understand the superior performance of this model given that it is simply a composition of linear operations. For this reason we restrict our analysis to those settings which give SoTA performance. To this end we ignore the LPF entirely in our subsequent analysis. While the LPF is an effective tool for compressing FITS, it comes with a performance degradation. 

In order to fully analyse FITS it is critical to introduce definitions for the Discrete and Real Fourier transforms.

\begin{definition}[Discrete Fourier Transform]\label{def:dft}
Let $\vec{x}\in \mathbb{R}^L$, we define the \textbf{Discrete Fourier Transform (DFT)} of $\vec{x}$ as $DFT_L:\mathbb{C}^L\rightarrow \mathbb{C}^L$ so that for $ j\in \{0,1,\ldots \floor{L} \}$
\begin{align}
    DFT_L(\vec{x})_j := \sum_{k=0}^{L-1} e^{\frac{-2\pi i kj}{L}}x_k
\end{align}
\end{definition}
The DFT can be written in matrix form $DFT_L(\vec{x}) = D_L\vec{x}$ where, if $\omega$ denotes the $L$\textsuperscript{th} root of unity ($\omega := e^{\frac{-2\pi i}{L}}$), then $D_L$ is the matrix:
\begin{align}\label{eqn:DL}
    D_L := \begin{bmatrix}
    1 & 1 & 1 & \ldots & 1\\
    1 & \omega & \omega^2 & \ldots & \omega^{L-1}\\
    \hdotsfor{5}\\
    1 & \omega^{L-1} & \omega^{2(L-1)} & \ldots & \omega^{L(L-1)}
\end{bmatrix}
\end{align}

The DFT is an invertible map and we define the \textbf{inverse DFT} (iDFT) by $iDFT(\vec{x}) := D_L^{-1}\vec{x}$ where $D_L^{-1} = \frac{1}{L}D_L^{\star}$.

FITS does not directly employ the DFT. Rather, it employs a closely related transform called the Real Fourier Transform, which we now define. 

\begin{definition}[Real Discrete Fourier Transform]\label{def:rfft}
Let $\vec{x}\in \mathbb{R}^L$, we define the \textbf{Real Discrete Fourier Transform (RFT)} of $\vec{x}$ as $RFT_L:\mathbb{R}^L\rightarrow \mathbb{C}^{\floor{L/2}+1}$ so that for $ j\in \{0,1,\ldots \floor{L/2} \}$
\begin{align}
    RFT_L(\vec{x})_j := \sum_{k=0}^{L-1} e^{\frac{-2\pi i kj}{L}}x_k
\end{align}
In other words, the RFT and DFT are identical other than the RFT is a truncated version that discards the last $\floor{L/2}-1$ components. The motivation for this comes from the fact that when $\vec{x}$ is real, then the $k$\textsuperscript{th} and $(L-k)$\textsuperscript{th} components of the DFT are complex conjugates ($DFT(\vec{x})_j = DFT(\vec{x})^{\star}_{L-j} $). For this reason these components contain the same information in that the original signal may be entirely reconstructed from the first $\floor{L/2}+1$ components via $(Y_0, Y_1, \ldots, Y_{\floor{L/2}+1})\mapsto iDFT(Re(Y_0), Y_1, Y_2, \ldots, Re(Y_{\floor{L/2}}), Y_{\floor{L/2}-1}^{\star}, \ldots , Y_1^{\star})$. We call this map the \textbf{inverse-RFT (iRFT)}. 

When $\vec{Y}$ has been obtained by taking the RFT of some real vector then $Y_0, Y_{\floor{L/2}}\in\mathbb{R}$ thus, taking the real part of these components, $Re(Y_0), Re(Y_{\floor{L/2}})$, does nothing. However writing the inverse like this allows us to also take the inverse RFT of complex vectors $\vec{Y}\in\mathbb{C}^{\floor{L/2}}$ which do not lie in the image $RFT(\mathbb{R}^L)$. 
\end{definition}

We can make the relationship between the DFT and RFT more explicit by defining the following linear map. 
\begin{definition}\label{def:pi}
Define the projection \(\Pi_L:\mathbb{C}^L \rightarrow \mathbb{C}^{\frac{L}{2}+1}\)
\[
\Pi_L(Y_0, Y_1, \ldots, Y_{L-1}) := (Y_0, Y_1, \ldots, Y_{\frac{L}{2}})
\]
Let \(\vec{Y} = (Y_0, Y_1, \ldots, Y_{\frac{L}{2}})\). Then, the inverse map \(\Pi^{-1}_L:\mathbb{C}^{\frac{L}{2}+1} \rightarrow \mathbb{C}^L\) is defined as
\[
\Pi^{-1}_L(\vec{Y}) = (Re(Y_0), Y_1, \ldots, Re(Y_{\frac{L}{2}}), Y^{\star}_{\frac{L}{2}-1}, \ldots, Y^{\star}_1)
\]
\end{definition}

Example: $\Pi_{L+T}^{-1}(Y_0 , Y_1 , Y_2 , Y_3) = (\frac{Y_0+Y_0^{\star}}{2} , Y_1 , Y_2 , \frac{Y_3+Y_3^{\star}}{2}, Y_2^{\star} , Y_1^{\star})$


\begin{remark}\label{remark:dft-pi}
Using this transformation one may express $RFT_L = \Pi_L \circ D_L$ and $iRFT_L = D_L^{-1}  \circ  \Pi^{-1}_L$
\end{remark}

\begin{remark}\label{remark:pi-inverse-property} For any $\vec{Y} \in DFT_L(\mathbb{R}^L)$ one may confirm that $\Pi^{-1}_L\circ \Pi_L = id_L$. Likewise, for any  $\vec{Y} \in RFT_L(\mathbb{R}^L)$ one may confirm that $\Pi_L\circ \Pi^{-1}_L = id_{\frac{L}{2}+1}$
\end{remark}


Having explicitly defined the discrete and discrete real Fourier transforms we are ready to begin the process of proving Theorem~\ref{theorem:fits} which we restate below. 

\begin{theorem} [FITS]
Let $M(\text{FITS})$ denote the FITS model class, i.e. the set of functions $f:\mathbb{R}^L\rightarrow\mathbb{R}^{T}$ which can be represented as a FITS model. When $L\geq T-2$, $M(\text{FITS})$ is precisely equal to the space of affine linear functions $A\vec{x} + \vec{b}$. 
\end{theorem}

As a composition of affine linear operation, FITS is itself an affine linear model. As a result any FITS model may be expressed in the form $A\vec{x} + \vec{b}$. The remainder of this section is dedicated to showing that when $L\geq T-2$ that $A$ and $\vec{b}$ are unconstrained meaning that FITS model class is equivalent to unconstrained linear regression. Before this we present a prescription, showing how one may obtain $A,\vec{b}$, given the complex bias and weight matrix from FITS's linear layer. 

\begin{lemma}\label{lemma:prescription}
    Let $f:\mathbb{R}^{L}\rightarrow\mathbb{R}^{L+T}$ be some FITS model. Let $W, \vec{c}$ be the weight matrix and bias of the complex linear layer in this model. Then we can express $f$ as a real affine linear map $f(x) = A\vec{x}+\vec{b}$ where $A = D_{L+T}^{-1}\circ\Pi_{L+T}^{-1} W\circ \Pi_L\circ D_L$ and $\vec{b} = iRFT(\vec{c})$
\end{lemma}
\begin{proof}
    As discussed before, as a composition of affine linear operation, FITS is itself an affine linear model. As a result any FITS model may be expressed in the form $A\vec{x} + \vec{b}$. One may recover the bias by applying $f$ to the zero vector by noting that $f(\vec{0})=A\vec{0}+\vec{b}=\vec{b}$. $f$ is a composition of the RFT, the complex affine map $\vec{x}\mapsto W\vec{x}+\vec{c}$ and an iRFT. Since the RFT maps zero to zero then $f(\vec{0})=iRFT(W\vec{0}+\vec{c}) = iRFT(\vec{c})$ as desired. 

    We know therefore that $A\vec{x} = A\vec{x} + \vec{b} - \vec{b} = iRFT(W\vec{z}+\vec{c}) -iRFT(\vec{c})$ where $\vec{z}:= RFT(\vec{x})$. Using the linearity of the iRFT we have $A\vec{x}= iRFT(W\vec{z}+\vec{c}-\vec{c}) = iRFT(W\vec{z}) = iRFT\circ W \circ RFT\vec{x}$. By writing the RFT and iRFT in terms of the DFT and the operator $\Pi$ as in Remark~\ref{remark:dft-pi} concludes our proof of Lemma~\ref{lemma:prescription}.
\end{proof}

\begin{mdframed}[linecolor=black,linewidth=1pt]
\textbf{Proof Structure for Theorem~\ref{theorem:fits}}: Our goal is to demonstrate that for $L\geq T-2$, any affine map $\vec{x}\mapsto A\vec{x}+\vec{b}$ can be represented using a FITS architecture. We achieve this by characterizing the set of matrices representable within a FITS framework. The detailed characterization is presented in Lemma~\ref{lemma:matrix_charac} and Lemma~\ref{lemma:fits_auxil}, which will be introduced subsequently. In Lemma~\ref{lemma:matrix_charac}, we introduce a specific set of linear maps, illustrating their formulation as complex matrix multiplications and detailing the process for deriving the corresponding matrix from the linear map. Lemma~\ref{lemma:fits_auxil} then ties these concepts directly to the FITS architecture, demonstrating how the linear map type discussed is integral to FITS. This establishes a comprehensive characterization of matrices expressible via FITS. Following these lemmas, we will prove that for $L\geq T-2$, this characterization includes all affine transformations $\vec{x}\mapsto A\vec{x}+\vec{b}$.
\end{mdframed}


In order to prove Theorem~\ref{theorem:fits} we must introduce the following set of complex matrices.
\begin{lemma}\label{lemma:matrix_charac}
    Let $\mathcal{A}$ denote the set of linear maps $T:DFT(\mathbb{R^L})\rightarrow iDFT(\mathbb{R}^{L+T})$ which can be expressed as a composition $T = \Pi_{L+T}^{-1}\circ W\circ \Pi_L$ where $W$ is some $\left(\frac{L+T}{2}+1\right)$ by $\left(\frac{L}{2}+1\right)$ complex matrix. We claim that each $T$ can be expressed as a complex matrix multiplication $T(W):DFT(\mathbb{R}^L)\rightarrow iDFT(\mathbb{R}^{L+T})$ where $T(W)$ is derived from $W$ as follows: 
    \begin{align}\label{eqn:matrix_conditions}
    (T(W))_{ij} = \left\{\begin{array}{lr}
        Re(W_{ij}), & i\in\{0,\frac{L+T}{2}\},  j= 0,\frac{L}{2}\\
        \frac{1}{2}(W_{ij}), & i\in\{0,\frac{L+T}{2}\},   0<j< L/2 \\
                \frac{1}{2}(W^{\star}_{i,L-j}), & i\in\{0,\frac{L+T}{2}\},   j> L/2 \\
        W_{ij}, & 0<i<\frac{L+T}{2}, j\leq\frac{L}{2} \\
        0, & 0<i<\frac{L+T}{2}, j>\frac{L}{2} \\
        W^{\star}_{L+T-i,j}, & i\notin \{0,\frac{L+T}{2} \}, j=0 \\
        W^{\star}_{L+T-i,L-j}, & \textit{ otherwise}
        \end{array}\right\} 
\end{align}

For example let $T=2$ and $L=4$ and let $W$ be an arbitrary complex $4\times 2$ matrix. Then one has:
\begin{align}\label{example:piBpi}
    T(W) = (\Pi_{L+T}^{-1}(W\Pi_L) ) =  \begin{bmatrix}
        Re({W}_{00}) & \frac{W_{01}}{2} & Re(W_{02}) & \frac{W^{\star}_{01}}{2} \\
        W_{10} & W_{11} & W_{12} & 0 \\
        W_{20} & W_{21} & W_{22} & 0 \\
        Re(W_{30}) & \frac{W_{31}}{2} & Re(W_{32}) & \frac{W^{\star}_{31}}{2} \\
        W^{\star}_{20} & 0 & W^{\star}_{22} & W^{\star}_{21} \\
        W^{\star}_{10} & 0 & W^{\star}_{12} & W^{\star}_{11} 
    \end{bmatrix}
\end{align}
\end{lemma}

\textbf{Remark}: In the statement of Lemma~\ref{lemma:matrix_charac}, $\Pi_{L+T}^{-1} \circ W \circ \Pi_L$ denotes the application of $\Pi_{L+T}^{-1}$ to each column of $W$ and applying $\Pi_{L}$ to each row. The order of these operations makes no difference since $\Pi_L$ is a projection.






\begin{proof}

Let $W$ be some arbitrary complex matrix of dimension $\left(\frac{L+T}{2}+1\right)$ by $\left(\frac{L}{2}+1\right)$. Let $T$ be the linear map defined on the domain $DFT(\mathbb{R^L})$ formed from the composition $\Pi_{L+T}^{-1} \circ W \circ \Pi_L$. We begin by assuming that there exists a complex $T+L$ by $L$ matrix $T(W)$ which is equivalent to this linear map and derive it's structure. At the end we then show that it is indeed equivalent to the linear map $T$. 

As a projection onto the first $1+\frac{L}{2}$ components, $\Pi_{L}$ can be written as an $(\frac{L}{2}+1) \times L$ matrix where the $ii$\textsuperscript{th} entry of $\Pi_{L}$ is  $1$ and all other entries are zero. Right composing $W$ by $\Pi_{L}$ yields a single matrix equivalent to appending $\frac{L}{2}-1$ columns of zeros to the right of $W$. That is:
\begin{align}\label{eqn:pi}
    (\Pi W)_{ij}= \left\{\begin{array}{lr}
        W_{ij}, & \text{for } j\leq \frac{L}{2}+1\\
        0, & \text{otherwise }
        \end{array}\right\}
\end{align}
For example, in the case $L=4, T=2$;

\begin{align*}
    W\circ \Pi_L &= \begin{bmatrix}
        W_{11} & W_{12} & W_{13} \\
        W_{21} & W_{22} & W_{23} \\
        W_{31} & W_{32} & W_{33} \\
        W_{41} & W_{42} & W_{43} 
    \end{bmatrix}
    \begin{bmatrix}
        1 & 0 & 0 & 0  \\
        0 & 1 & 0 & 0\\
        0 & 0 & 1 & 0
    \end{bmatrix} \\
    &= \begin{bmatrix}
        W_{11} & W_{12} & W_{13} & 0 \\
        W_{21} & W_{22} & W_{23} & 0 \\
        W_{31} & W_{32} & W_{33} & 0 \\
        W_{41} & W_{42} & W_{43} & 0 
    \end{bmatrix}
\end{align*}

Now, let $B = (W\circ \Pi_L)$ be some complex $(\frac{T+L}{2}\times L)$ matrix. One may similarly write $\Pi_{L+T}^{-1}\circ  B$, as a single matrix $D \in\mathbb{C}^{((T+L)\times L)}$. Using the definition of $\Pi_{T+L}^{-1}$ (Definition~\ref{def:pi}) we can derive the entries of the matrix $D$. We do this by noting that the matrix $D$ must satisfy $D\vec{Y}=\Pi_{L+T}^{-1}\circ  B$ for any $\vec{Y}\in  DFT(\mathbb{R}^L)$. By equating the components $(D\vec{Y})_i=(\Pi_{L+T}^{-1}\circ  B)_i$, one may deduce the matrix $D$ in terms of $B$.

\textbf{Case 1:} Let $i=0,\frac{L+T}{2}$. 

$(D\vec{Y})_i = (\Pi^{-1}_{L+T} B\vec{Y})_i := Re((B\vec{Y})_i) = \frac{(B\vec{Y})_i + (B\vec{Y})^{\star}_i}{2}$. Therefore,
\begin{align*}
    (D\vec{Y})_i = \sum_{j=0}^{L-1} D_{ij}Y_j  =  \frac{1}{2}\bigg(\sum_{j=0}^{L-1} B_{ij}Y_j + B^{\star}_{ij}Y^{\star}_j\bigg) \\
    = \bigg(\frac{B_{i0}+B_{i0}^{\star}}{2}\bigg)Y_0 + \sum_{j=1}^{L-1} \bigg(\frac{B_{ij} + B^{\star}_{i,L-j}}{2}\bigg)Y_j 
\end{align*}

As this holds for all $\vec{Y}\in DFT(\mathbb{R}^L)$, we may conclude; $D_{i0} = Re(B_{i0}), D_{i,L/2} = Re(B_{i,L/2})$ and otherwise; $D_{ij} = \bigg(\frac{B_{ij} + B^{\star}_{i,L-j}}{2}\bigg)$. 

\textbf{Case 2:} Let $0<i<\frac{L+T}{2}$.

By the definition of $\Pi^{-1}_{L+T}$ we have $(D\vec{Y})_i=(\Pi_{L+T}^{-1}  B\vec{Y})_i = (B\vec{Y})_i$. Since this holds for all $\vec{Y}$ we must have $D_{ij} = B_{ij}$ for all $j$.

\textbf{Case 3:} Let $i > \frac{L+T}{2}$.

Using Def.~\ref{def:pi} we may derive the following which holds for all $\vec{Y}\in DFT(\mathbb{R}^L)$:
\begin{align*}
   (D \vec{Y})_{i} = (\Pi^{-1}_{L+T} B\vec{Y})_i = (B\vec{Y})^{\star}_{L+T-i}  \\
    \implies (\sum_{j=0}^{L-1}D_{ij}Y_j) = \sum_{j=0}^{L-1}B^{\star}_{L+T-i,j}Y^{\star}_j\\ =  (\sum_{j=1}^{L-1}B^{\star}_{L+T-i,j}Y_{L-j}) + B^{\star}_{L+T-i,0}Y_0
\end{align*}
Since $\vec{Y}\in DFT(\mathbb{R}^L)$ we know that $Y_0, Y_{L/2}\in\mathbb{R}$ and otherwise $Y_j = Y^{\star}_{L-j}$. It follows therefore that $D_{i0} = B_{L+T-i,0}^{\star}$ and $D_{ij} = B_{L+T-i,L-j}^{\star}$ for $j>0$.

Below we summarise our findings, writing a general expression for the $ij$\textsuperscript{th} component of $D=\Pi^{-1}_{L+T} B$

\begin{align}\label{eqn:pi-inverseB}
    (\Pi^{-1}_{L+T}B)_{ij}= \left\{\begin{array}{lr}
        Re(B_{ij}), & i\in\{0,\frac{L+T}{2}\},  j= 0\\
        \frac{1}{2}(B_{ij}+B_{i,L-j}^{\star}), & i\in\{0,\frac{L+T}{2}\},   j\neq 0 \\
        B_{ij}, & 0<i<\frac{L+T}{2} \\
        B^{\star}_{L+T-i,j}, & i\notin \{0,\frac{L+T}{2} \}, j=0 \\
        B^{\star}_{L+T-i,L-j}, & \textit{ otherwise}
        \end{array}\right\} 
\end{align}

We may combine Equation~\ref{eqn:pi-inverseB} with the earlier Equation~\ref{eqn:pi} to establish a general form for the $ij$\textsuperscript{th} element of $\Pi^{-1}_{L+T}W\Pi_L$:

\begin{align}\label{eqn:matrix_conditions_derived}
   T(W)_{ij} := (\Pi^{-1}_{L+T}W\Pi_L)_{ij} = \left\{\begin{array}{lr}
        Re(W_{ij}), & i\in\{0,\frac{L+T}{2}\},  j= 0,\frac{L}{2}\\
        \frac{1}{2}(W_{ij}), & i\in\{0,\frac{L+T}{2}\},   0<j< L/2 \\
                \frac{1}{2}(W^{\star}_{i,L-j}), & i\in\{0,\frac{L+T}{2}\},   j> L/2 \\
        W_{ij}, & 0<i<\frac{L+T}{2}, j\leq\frac{L}{2} \\
        0, & 0<i<\frac{L+T}{2}, j>\frac{L}{2} \\
        W^{\star}_{L+T-i,j}, & i\notin \{0,\frac{L+T}{2} \}, j=0 \\
        W^{\star}_{L+T-i,L-j}, & \textit{ otherwise}
        \end{array}\right\} 
\end{align}

This is precisely the characterisation given in Equation~\ref{eqn:matrix_conditions_derived}. 

\textbf{Existence Proof:} We have demonstrated the necessary structure for a complex matrix $T(W)$ that is equivalent to the linear map $\Pi_{L+T}^{-1} \circ W \circ \Pi_L$. The task now is to prove that such a complex matrix representation, $T(W)$, indeed exists. 

The map $\Pi_{L+T}^{-1} \circ W \circ \Pi_L$, being real-linear, can be represented by a real matrix $M$ when considering its domain, $DFT(\mathbb{R}^L)$, as a real vector space of dimension $2L$. The domain $DFT(\mathbb{R}^L)$ consists of complex vectors $(z_0, z_1, \ldots, z_{L/2}, \ldots, z_{L-1})$, with $z_0$ and $z_{L/2}$ real, and $z_i = z^*_{L-i}$ for other indices, indicating complex conjugate pairs.

To transition from a real to a complex matrix representation, we exploit the structure of these complex vectors by expressing the real and imaginary parts of $z_i$ as $\text{Re}(z_i) = \frac{z_i + z^*_i}{2}$ and $\text{Im}(z_i) = \frac{z_i - z^*_i}{2}$, respectively. This approach allows for the real matrix $M$, defined in terms of the real and imaginary components, to be reformulated as a complex matrix, thus confirming the existence and formulation of $T(W)$ as a complex matrix multiplication.
\end{proof}


\begin{lemma}\label{lemma:fits_auxil}
Any FITS model can be expressed in the form $\vec{x}\mapsto A\vec{x} + \vec{b}$ where $ A \in D_{L+T}^{-1}\circ \mathcal{A}  \circ D_{L}$ and $\vec{b} \in \mathbb{R}^{L+T}$. Here $\mathcal{A}$ denotes the set of matrices introduced in Lemma~\ref{lemma:matrix_charac}. Conversely, if $\vec{x}\mapsto A \vec{x} + \vec{b} $ is affine linear map such that $A \in D_{L+T}^{-1}\circ \mathcal{A} \circ D_{L}$, then there exists a functionally equivalent FITS model.
\end{lemma}
\begin{proof}
We reiterate, that as a sequence of affine linear operations, FITS is a real affine linear model $\mathbb{R}^{L}\rightarrow \mathbb{R}^{L+T}$. It follows that any FITS model can be expressed in the form $A\vec{x}+\vec{b}$ for some choice of $A\in \mathbb{R}^{(L+T)\times L}$ and $\vec{b}\in \mathbb{R}^{L}$. It remains to show that for any FITs model, $A$ can be selected from the family $D_{L+T}^{-1}\circ \mathcal{A}  \circ D_{L}$. 

We showed in Lemma~\ref{lemma:prescription} that, if $W,\vec{c}$ are the weights matrix and bias of the complex linear layer in the FITS model then $\vec{b} = iRFT(\vec{c})$ and $A = D_{L+T}^{-1}\circ\Pi_{L+T}^{-1} W\circ \Pi_L\circ D_L$.

Putting this together, we have
\begin{align*}
    FITS(\vec{x}; W, \vec{c})  &= D_{L+T}^{-1}(\Pi_{L+T}^{-1}W \Pi_{L})D_{L}\vec{x} + iRFT(\vec{c})
\end{align*}
which can be compactly expressed as 
\[
FITS(\vec{x}; W, \vec{c}) = D_{L+T}^{-1} B D_{L}
\]
where \( B = \Pi_{L+T}^{-1}W \Pi_{L} :DFT(\mathbb{R}^L)\rightarrow iDFT(\mathbb{R}^{L+T})\) belongs to \( \mathcal{A} \) as desired.

Since $W$ can be any complex matrix then the converse also holds in that any linear map $\vec{x}\mapsto A\vec{x}$ where $A\in D_{L+T}^{-1} \mathcal{A} D_{L}$ must be equivalent to a FITS model. 

It remains to show that any bias $\vec{b}\in\mathbb{R}^{L+T}$ can be expressed in the form $iRFT(\vec{c})$ where $\vec{c}\in\mathbb{C}^{\frac{L+T}{2}+1}$. This follows from the bijectivity of the DFT. Specifically, we can obtain any $\vec{b}$ by letting $\vec{c}:= RFT(\vec{b})$. Then $iRFT(RFT(\vec{b})) = iDFT\circ \Pi_{L+T}^{-1}\circ \Pi_{L+T}\circ DFT (\vec{b}) = \vec{b}$ by Remark~\ref{remark:pi-inverse-property}.

\end{proof}

Having proved Lemma~\ref{lemma:fits_auxil} and the more technical Lemma~\ref{lemma:matrix_charac} we are now ready to prove Theorem~\ref{theorem:fits}. 

\textbf{Proof of Lemma~\ref{theorem:fits}} 
\begin{proof}
We showed in Lemma~\ref{lemma:fits_auxil} that every FITS model $\vec{x}\mapsto FITS(\vec{x};W,\vec{c})$ can be written in the form $x\mapsto A\vec{x}+\vec{b}$ where $A\in \mathbb{R}^{(L+T)\times L}$, $\vec{b}\in \mathbb{R}^{L+T}$. Moreover we showed how one may obtain $A, \vec{b}$ from $W, \vec{c}$ via $A=D_{L+T}^{-1}\Pi^{-1}_{L+T}W\Pi_L D_L$ and $\vec{b}=iRFT\vec{c}$. FITS outputs both a forecast and a reconstruction of the context. Consequently we may decompose $A = \begin{bmatrix}
    A_L \\ A_T
\end{bmatrix}$ where $A_T\in \mathbb{R}^{T\times L}$ is the matrix which produces a forecast from the context vector. We have already seen in Lemma~\ref{lemma:fits_auxil}  that FITS imposes no restriction on our bias term $\vec{b}$. Our claim is that additionally, when $L\geq T-2$, any real ${T\times L}$ matrix $A_T$ by be attained an appropriate selection of $W$. If we define the operator $P:\mathbb{R}^{(T+L)\times L} \rightarrow \mathbb{R}^{T\times L}$ by $P\bigg(\begin{bmatrix}
    A_L \\ A_T
\end{bmatrix}\bigg) = A_T$ then we can formulate this claim as
\begin{align*}
   L\geq T-2 \implies \mathbb{R}^{T\times L} \subseteq P\circ D_{L+T}^{-1}\circ\mathcal{A}\circ D_L
\end{align*}
Since $D_L$ is bijective this may be equivalently be written as 
\begin{align*}
    \mathbb{R}^{T\times L}\circ D_L^{-1} \subseteq P\circ D_{L+T}^{-1}\circ\mathcal{A}
\end{align*}
Note that we already have the reverse inclusion $P\circ D_{L+T}^{-1}\circ\mathcal{A}\circ D_L \subseteq \mathbb{R}^{T\times L}$ by Lemma~\ref{lemma:fits_auxil}. 

We begin by characterising the space of matrices $\mathbb{R}^{T\times L}\circ D_L^{-1}$. This is the space of matrices one gets when you apply an inverse DFT to the rows of all real $T\times L$ matrices. That is, $\mathbb{R}^{T\times L}\circ D_L^{-1}$ is the subset of complex $T\times L$ matrices where each row is in the set $D_L^{-1}(\mathbb{R}^L)$. These are precisely the complex vectors of length $L$ where $v_0, v_{L/2} \in \mathbb{R}$ and where otherwise $v_{i} = v^{\star}_{L-i}$. For example, for $T=6, L=4$ the general form of $\mathbb{R}^{T\times L}\circ D_L^{-1}$ can be written as follows where lowercase denotes a real entry. 
\begin{align}
    \begin{bmatrix}
        b_{00} & B_{01} & b_{02} & B^{\star}_{01} \\
        b_{01} & B_{11} & b_{12} & B^{\star}_{11} \\
        b_{02} & B_{21} & b_{22} & B^{\star}_{21} \\
        b_{03} & B_{31} & b_{32} & B^{\star}_{31} \\
        b_{04} & B_{41} & b_{42} & B^{\star}_{41} \\
       b_{05} & B_{51} & b_{52} & B^{\star}_{51} 
    \end{bmatrix}
\end{align}

Using this fact, $\mathbb{R}^{T\times L}\circ D_L^{-1}$ may be alternatively be characterised as the set of complex $T\times L$ matrices where the zeroth and $L/2$\textsuperscript{th} columns, $c_0, c_{L/2}$, are arbitrary real vectors and where otherwise all other columns are arbitrary complex vectors subject to the condition $c_{i} = c_{L-i}^{\star}$. Written as a vector space isomorphism this is:
\begin{align*}
    \mathbb{R}^{T\times L}\circ D_L^{-1}\cong \mathbb{R}^T \oplus \underbrace{\mathbb{C}^T \oplus \ldots \oplus \mathbb{C}^T}_{(\frac{L}{2} - 1) \text{ times}} \oplus \mathbb{R}^T \oplus\underbrace{\mathbb{C}^T \ldots \oplus \mathbb{C}^T}_{(\frac{L}{2} - 1) \text{ times}}
\end{align*}

Using Lemma~\ref{lemma:matrix_charac} one may characterise $\mathcal{A}$ similarly in terms of its columns. The zeroth and $L/2$\textsuperscript{th} columns are arbitrary vectors in $D_{L+T}(\mathbb{R}^{L+T})$. For $0<i<L/2$ the $i$\textsuperscript{th} column, $c_i$, is an arbitrary complex vector of length $T+L$, subject to the condition $c_{ij}=0$ for $j> (T+L)/2$. For $i>L/2$ column $c_i$ satisfies the condition $c_{i0} = c_{L-i,0}$ and otherwise $c_{ij} = c^{\star}_{L-i,L+T-j}$. In the case $T=2, L=4$ the general form for a matrix in $\mathcal{A}$ can be written as follows where lowercase once again denotes a real entry:
\begin{align*}
       \begin{bmatrix}
        a_{00} & A_{01} & A_{02} & a_{03} & A^{\star}_{02} & A^{\star}_{01}  \\
        A_{10} & A_{11} & A_{12} & A_{13}      & 0              &  0 \\
         A_{20} & A_{21} & A_{22} & A_{23}     &  0              & 0 \\
          a_{30} & A_{31} & A_{32} & a_{33} & A^{\star}_{32} & A^{\star}_{31} \\
           A^{\star}_{20} & 0 & 0 & A^{\star}_{23}     & A^{\star}_{22}        & A^{\star}_{21} \\
           A^{\star}_{20} & 0 & 0 & A^{\star}_{13}     &  A^{\star}_{12}        & A^{\star}_{11} \\
    \end{bmatrix} 
\end{align*}

We will use the notation $\mathcal{S}$ to represent the space of complex vectors $\vec{v}\in \mathbb{C}^{(L+T)}$ where $v_i=0$ for $i>(T+L)/2$.
\begin{align*}
    \mathcal{S} := \{\vec{v}\in \mathbb{C}^{(L+T)} | v_i=0, i>(T+L)/2 \}
\end{align*}
Using this, one may write $\mathcal{A}$ as a vector isomorphism.
\begin{align*}
    \mathcal{A} \cong D_{L+T}(\mathbb{R}^{L+T}) \oplus \underbrace{\mathcal{S} \oplus \ldots \oplus \mathcal{S}}_{(\frac{L}{2} - 1) \text{ times}} \oplus D_{L+T}(\mathbb{R}^{L+T}) \oplus\underbrace{\mathcal{S}\oplus  \ldots \oplus \mathcal{S}}_{(\frac{L}{2} - 1) \text{ times}}
\end{align*}

Remark: Note that in both cases $\mathcal{A}$ and $\mathbb{R}^{T\times L}\circ D_L^{-1}$ are completely specified by their first $(L/2)$+1 columns since the for $i>L/2$ column $c_i$ can be determined completely by $c_{L-i}$

 If $A_T$ is some arbitrary matrix in $\mathbb{R}^{T\times L}\circ D_L^{-1}$ we want to show that, when $L\geq T-2$, we can find $W\in \mathcal{A}$ where $P\circ D_{L+T}^{-1}(W) = A$. 

We observe that the linear map $P\circ D_{L+T}^{-1}:\mathbb{C}^{L+T}\rightarrow \mathbb{C}^T$ operates independently on each column of $\mathcal{A}$. Thus, using the decompositions given above for $\mathcal{A}$ and $D_{L+T}(\mathbb{R}^{L+T})$, we only need to show that: 
  \begin{align*}
     \mathbb{R}^T \subseteq P\circ D_{L+T}^{-1}( D_{L+T}(\mathbb{R}^{L+T}))
\end{align*}
 and
\begin{align*}
      \mathbb{C}^T \subseteq P\circ D_{L+T}^{-1}(\mathcal{S})
  \end{align*}

The former of these is trivial since $D_{L+T}$ is a bijection; hence $P\circ D_{L+T}^{-1}(D_{L+T}(\mathbb{R}^{L+T})) = P(\mathbb{R}^{L+T}) = \mathbb{R}^T$. To show the second inclusion, we recollect that we already have $P\circ D_{L+T}^{-1}(\mathcal{S}) \subseteq \mathbb{C}^T$. Hence, we need only to show that the dimension of the space $P\circ D_{L+T}^{-1}(\mathcal{S})$ is greater than $T$ (the dimension of $\mathbb{C}^T$).

We claim that $\dim(P\circ D_{L+T}^{-1}(\mathcal{S})) = \min(T, \frac{T+L}{2}+1)$. Hence, we have $\dim(P\circ D_{L+T}^{-1}(\mathcal{S})) \geq T \iff \frac{T+L}{2}+1 \geq T \iff L \geq T-2$ as required.

In order to demonstrate this claim, note that $P\circ D^{-1}{L+T}$ can be written as a $T \times (L+T)$ matrix formed by taking the bottom $T$ rows of the matrix $D{L+T}^{-1}$. Then, due to the structure of $\mathcal{S}$ (namely, that $v_i=0$ for all $i > \frac{T+L}{2}$), $\dim(P\circ D^{-1}{L+T}(\mathcal{S}))$ is equal to the rank of the $T \times \left(\frac{T+L}{2}+1\right)$ submatrix extracted from the bottom left of the $(L+T) \times (L+T)$ matrix $D{L+T}^{-1}$. If we can show that this submatrix has full row and column rank, then we are done. Let $a := \min\left(\frac{T+L}{2}+1, T\right)$ and form the squared $a \times a$ matrix by discarding the excess rows or columns. We claim this square matrix has rank $a$. This follows from the fact that this submatrix is a Vandermonde matrix generated from a root of unity, thus it has a non-zero Vandermonde determinant and is therefore full rank.

\end{proof}




\section{Further Results and Experiments}
 \begin{figure}[!htbp]
	\centering
	\includegraphics[width=0.5\linewidth]{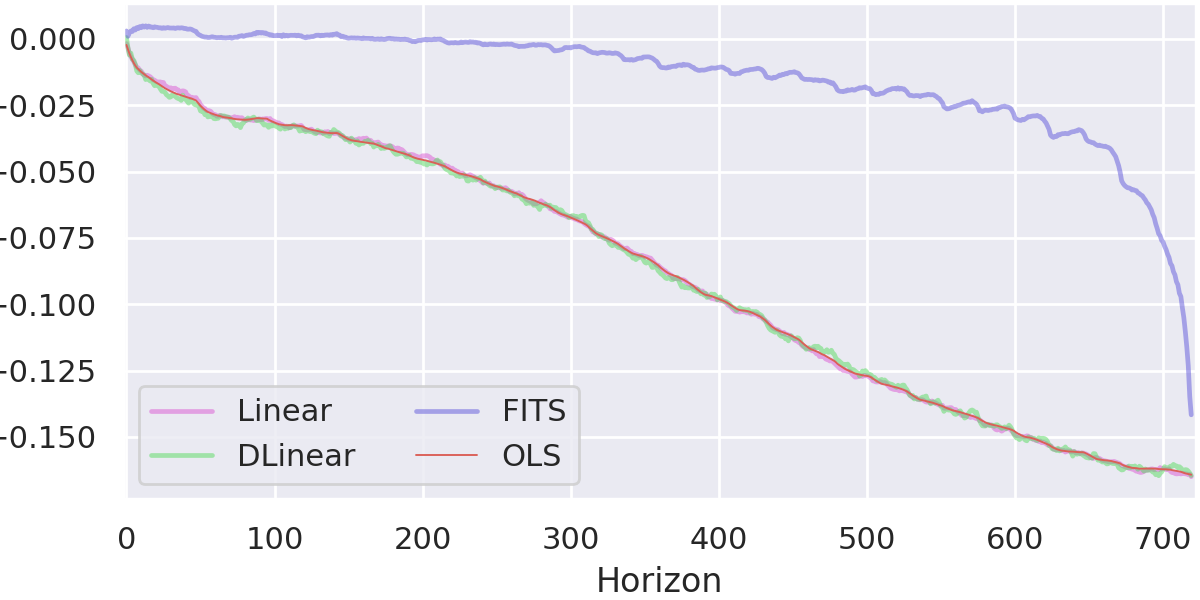}
	\caption{The biases learned by the FITS, Linear, DLinear after being trained on ETTh1 for 50 epochs. We also include the bias learned by the closed-form OLS linear regression. We note that, in line with theory from Section~\ref{sec:model_analysis}, we get the same bias for the DLinear, OLS and Linear models. Notably the bias for FITS is substantially different. This is explained by the choice of normalisation used in the Fourier transform in FITS. \label{fig:noRevIN_biases}
 }
\end{figure}

\section{FITS Bias Term - Detailed Breakdown}\label{sec:fits-bias}
In this section we explain and breakdown Section~\ref{sec:fitsBias-main} explaining how FITS operates as an almost bias-free model early in training.

In Def~\ref{def:dft} we defined the DFT and its inverse. The definition we use is standard and in line the implementation used in FITS \footnote{\url{https://github.com/VEWOXIC/FITS/}}. An alternative definition of the DFT instead defines it as follows:
\begin{align}
    DFT_L(\vec{x})_j := \frac{1}{\sqrt{N}}\sum_{k=0}^{L-1} e^{\frac{-2\pi i kj}{L}}x_k
\end{align}
The inverse DFT is then defined as $\frac{1}{\sqrt{N}}D^{*}_L$ where $D_L$ is as defined in Eqn.~\ref{eqn:DL}.

This second definition is identical, differing only in the choice of normalisation.  This alternative definition is referred to as the \textit{Normalised } or \textit{Orthonormallly Normalised} DFT. A key property of the normalised DFT is that it is distance-preserving. By this we mean that $\vert\vert DFT(\vec{x})\vert\vert_2 = \vert\vert \vec{x}\vert\vert_2 $. On the contrary, the DFT as it is defined in Def.~\ref{def:dft} satisfies $\vert\vert DFT(\vec{x})\vert\vert_2 = \sqrt{N}\vert\vert \vec{x}\vert\vert_2$ where $N$ is the number of components of the vector $\vec{x}$. In other words, the DFT as defined in Section~\ref{sec:fits} stretches distances by a factor of $\sqrt{N}$. The opposite of this is true for the inverse DFT so that $\vert\vert iDFT(\vec{x})\vert\vert_2 = \frac{1}{\sqrt{N}}\vert\vert \vec{x}\vert\vert_2$.

We saw in Lemma~\ref{lemma:prescription} that any FITS model may be expressed in the form $A\vec{x}+\vec{b}$. Moreover if $\vec{c}$ denotes the bias in FITS's complex linear layer then we may obtain $\vec{b}$ via $\vec{b}=iRFT(\vec{c})$. Since the $iRFT$ is real-linear this means that $\vec{b}$ and $\vec{c}$ are related via a matrix equation $\vec{b} = M\vec{C}$ where $\vec{C}$ is the real vector obtained by splitting $\vec{c}$ into its real and imaginary components $\vec{C}:=\begin{bmatrix}
    \vec{Re(c)}\\ \vec{Im(c)}
\end{bmatrix}$

 Let us now consider what the ramifications are of learning $\vec{b}$ by stochastic gradient descent (SGD) using the parameterisation $\vec{b} = M\vec{C}$ rather than learning $\vec{b}$ directly. If $\Bar{v}$ denotes the derivative of the loss with respect to the variable $\vec{v}$: that is $\bar{v}:=\frac{\partial L}{\partial \vec{v}}$ and $\eta$ denotes our learning rate then the gradient update using a naive parameterisation of $\vec{b}$ is:
\begin{align*}
    \vec{b}\mapsto \vec{b} - \eta \bar{b}
\end{align*}

Conversely, if we let $\vec{b} = M\vec{C}$ and we instead learn $\Vec{C}$ by gradient descent. One may show by the chain rule that $\bar{C}: =  M^T\bar{b}$. Thus, using the same learning rate as before, this induces an update 
\begin{align*}
    \vec{C}\mapsto \vec{C} - \eta \bar{C} = \vec{C} - \eta M^T\bar{b}\\
    \implies \vec{b}\mapsto \vec{b} -  \eta MM^T\bar{b}\\
\end{align*}
Thus, this choice of parameterisation means that we get an update of $MM^T\bar{b}$ where naively we would have an update of $\bar{b}$.  

It should be immediately clear, that unless $MM^T$ is approximately distance preserving that we are effectively scaling the learning rate of our bias $\vec{b}$. As we have discussed, because we are using a non-orthonormal normalisation $M$ scales $\vec{c}$ in the order of $\frac{1}{\sqrt{L+T}}$. Put together, this means that $MM^T$ is scaling $\vec{c}$ in the order of $\frac{1}{L+T}$. FITS applies a scaling of $\frac{L+T}{L}$ before outputting the forecast which partially mitigates this. However in conclusion, $\vec{b}$ still has a learning rate approximately $\frac{1}{L}$ times smaller than one would obtain through a naive parameterisation of $\vec{b}$.

As we saw, any FITS model can be expressed in the form $A\vec{x}+\vec{b}$. It is natural to ask whether this phenomena also impacts the weight matrix $A$. In fact it does not. As a result the issue with the bias cannot simply be resolved by increasing the learning rate as this would result in a learning rate which is too high for learning $A$. Let briefly sketch the reason why the weight matrix doesn't also have these issues. Crudely speaking the weight matrix $W$ in FITS's complex linear layer and $A$ are related via an expression of the $A= M_1WM_2$ where $M_1, M_2$ are real matrices corresponding to the iRFT and RFT respectively. If one chooses to normalise the iRFT by $\frac{1}{N}$ this is then offset by the fact that the right multiplication $M_2$ is unnormalised. In terms of backpropagation rules we have:
\begin{align*}
    \bar{W} = M_1^T \bar{A} M_2^T
\end{align*}
Therefore, whereas under a naive parameterisation we would have an update of $\bar{A}$, FITS gives us an update of $M_1M_1^T \bar{A} M_2^TM_2$. Thus, whatever normalisation standard we use for the RFT; whether we normalise the RFT $(M_2)$ by $\frac{1}{N}$ but not the iRFT $(M_1)$ or whether we normalise them both equally, leads the same update. 

\section{Further Proofs}\label{App:futher_proofs}
\subsection{DLinear}
In Lemma~\ref{lemma:dlinear} we write the padded moving average, utilised in DLinear to obtain the trend of $\vec{x}$, as a matrix multiplication $D\vec{x}$. In this part we explain the structure of $D$. We do this by means of an example: Consider the simple case where we have a context vector $\vec{x}$ of length 6 and we take a moving average with a kernel size of 3. In order to preserve dimension of $\vec{x}$ on must pad either side of $\Vec{x}$. We do this by repeating the first and last values twice before applying the moving average. That is:
\begin{align*}
    (x_1, x_2, x_3, x_4, x_5, x_6) \mapsto (x_1,x_1, x_2, x_3, x_4, x_5, x_6, x_6)
\end{align*}
In general if we have a kernel size of $K$ where $K$ is odd then we must pad each side with $\frac{K-1}{2}$ repeated entries. (In DLinear they use a kernel size of 25 \cite{zeng2023transformers}).

This padding operation can be expressed in matrix form. For this example
\begin{align*}
    \begin{bmatrix}
        x_1\\x_1\\ x_2\\ x_3\\ x_4\\ x_5\\ x_6\\ x_6
    \end{bmatrix}
    =
    \begin{bmatrix}
        1 & 0 & 0 & 0 & 0 & 0 \\
        1 & 0 & 0 & 0 & 0 & 0 \\
        0 & 1 & 0 & 0 & 0 & 0 \\
        0 & 0 & 1 & 0 & 0 & 0 \\
        0 & 0 & 0 & 1 & 0 & 0 \\
        0 & 0 & 0 & 0 & 1 & 0 \\
        0 & 0 & 0 & 0 & 0 & 1 \\
        0 & 0 & 0 & 0 & 0 & 1 \\
    \end{bmatrix}
        \begin{bmatrix}
        x_1\\ x_2\\ x_3\\ x_4\\ x_5\\ x_6
    \end{bmatrix}
\end{align*}

The moving average of the expanded $\vec{x}$ is calculated by taking the arithmetic mean of each successive run of $K=3$ values. This may be written as a matrix multiplication:
\begin{align}
\frac{1}{3}
    \begin{bmatrix}
        1 & 1 & 1 & 0 & 0 & 0 & 0 & 0\\
        0 & 1 & 1 & 1 & 0 & 0 & 0 & 0\\
        0 & 0 & 1 & 1 & 1 & 0 & 0 & 0 \\
        0 & 0 & 0 & 1 & 1 & 1 & 0 & 0\\
        0 & 0 & 0 & 0 & 1 & 1 & 1 & 0\\
        0 & 0 & 0 & 0 & 0 & 1 & 1 & 1 \\
    \end{bmatrix}
    \begin{bmatrix}
        x_1\\x_1\\ x_2\\ x_3\\ x_4\\ x_5\\ x_6\\ x_6
    \end{bmatrix}
\end{align}
We can combine these two operations into a single matrix multiplication to obtain
\begin{align*}
D\vec{x} := 
\frac{1}{3}
        \begin{bmatrix}
        2 & 1 & 0 & 0 & 0 & 0 \\
        1 & 1 & 1 & 0 & 0 & 0 \\
        0 & 1 & 1 & 1 & 0 & 0 \\
        0 & 0 & 1 & 1 & 1 & 0 \\
        0 & 0 & 0 & 1 & 1 & 1 \\
        0 & 0 & 0 & 0 & 1 & 2 \\
    \end{bmatrix}
        \begin{bmatrix}
        x_1\\ x_2\\ x_3\\ x_4\\ x_5\\ x_6
    \end{bmatrix}
\end{align*}

\subsection{Closed Form Solution to Linear Regression}\label{sec:closed_form}
A well-known property of least-squares linear regression is that it admits a closed-form solution \cite{hastie2009elements}. There are a number of ways in which one may compute this solution numerically. Below we define one of the more common approaches:  

\begin{definition}[Closed-Form Solution]\label{def:closed:form}
Let \( X \) denote the \( N \times L \) design matrix containing our training data, and let \( Y \) denote the \( N \times T \) matrix of training targets. The \( L \times T \) weight matrix \( W \) that minimises the training loss \( \Vert XW - Y \Vert_2^2 \) is given in closed form as follows:
\begin{align}
    W^{\star} = (X^TX)^{-1}X^TY
\end{align}
If the rank of \( X \) is less than \( L \), indicating that \( X \) is rank-deficient, a unique solution may not exist. In such cases, a solution can be obtained using the Moore-Penrose pseudo-inverse, denoted as \( (X^TX)^{+} \), instead of the regular inverse.
\end{definition}

In practice the solution given in Def.~\ref{def:closed:form} may be numerically unstable if $X^TX$ is ill-conditioned. In Section~\ref{sec:experiments} we use the more stable but more expensive SVD approach. Details of this may be found in \citet{Silva_SVDRegression}.

\subsubsection{Closed Form Solutions for Linear Regression plus Data Normalisation}\label{sec:closed_form_data_norm}

Standard least-squares linear regression admits a closed-form solution. We claimed in Section~\ref{sec:discussion} that the other model families in Table~\ref{table:model_summary}, corresponding to RLinear and NLinear also admit a closed form solution under a least-squares loss. The reason for this is that one can formulate each of these models as linear regression on a suitable feature set of $\vec{x}$. 

In the following we will let $X,Y$ be $N\times L$ and $N\times T$ denote matrices containing $N$ training samples and their targets respectively.

\textbf{NLinear}: Suppose that we wish to find the matrix $\tilde{A}\in\mathbb{R}^{T \times L}$ and bias $\vec{b}\in\mathbb{R}^{T}$ which minimises $\vert\vert AX^T+\vec{b}-Y^T\vert\vert_2$ subject to the condition that the rows of $A$ must sum to one. Augment $X$ and $Y$ by computing the row-mean of $X$ and subtracting this off each of the rows of both $X$ and $Y$. We denote the augmented matrices as $\tilde{X}$ and $\Tilde{Y}$. Specifically, for row $i$; $\tilde{X}_i:=X_i - \mu(X_i)$ and $\tilde{Y}_i:=Y_i - \mu(X_i)$. We now solve the unconstrained least-squares regression on these augmented matrices. This yields a matrix $A^*$ and a bias $\vec{b}$ which minimises $\vert\vert A \tilde{X}^T+\vec{b}-\tilde{Y}^T \vert\vert_2$. This matrix is not uniquely defined since one may add any multiple of the vector $(1,1,1,\ldots, 1)$ to each row and obtain the same train loss since $\vec{1}\tilde{X} = \vec{0}$. Thus, we can choose to project $A$ so that each of it's rows do sum to 1. We claim that this matrix minimises our original constrained objective $\vert\vert AX^T+\vec{b}-Y^T\vert\vert_2$. 

Let $\vec{x},\vec{y}$ be an arbitrary context-target pair. Let $\mu(\vec{x})_k$ be notation for a $k$-dimensional vector formed by taking the mean of $\vec{x}$ and repeating this $k$-times. Since the rows of of $A$ sum to one then $A\mu(X)_L = \mu(X)_T$. Therefore:
\begin{align*}
   \vert\vert (A\vec{x}+\vec{b}-\vec{y})\vert\vert_2 &= \vert\vert (A\vec{x}+\vec{b}-\vec{y})+ \mu(\vec{x})_T^T-\mu(\vec{x})_T^T \vert\vert_2  \\ &=\vert\vert (A\vec{x}+\vec{b}-\vec{y})+ \mu(\vec{x})_T^T-(A\mu(\vec{x})_L)^T \vert\vert_2 \\
    &= \vert\vert (A(\vec{x}-\mu(\vec{x})_L)+\vec{b}-(\vec{y}-\mu(\vec{x})_T)\vert\vert_2 
\end{align*}
Therefore, any matrix $A$ will get the same MSE on any pair $\vec{x}, \vec{y}$ as it will on the augmented versions $\vec{x}-\mu(\vec{x})_L$, $\vec{y}-\mu(\vec{x})_T)$. It follows that $A^*$ obtains the same MSE on $\tilde{X},\tilde{Y}$ as $X,Y$ and vice versa. In particular if $A^*$ is also optimal for $\tilde{X},\tilde{Y}$ then it is too for $X,Y$. Thus, the matrix which we obtained by closed-form OLS on $\tilde{X},\tilde{Y}$ satisfies the properties claimed.

\textbf{RLinear}: We showed that one can find a global optima for the NLinear models class in closed form when using a mean-squares loss function. The same is true for RLinear and the constuction is much the same. We wish to find a matrix $A$ and bias $\vec{b}$ which minimise $\vert\vert AX^T+\vec{b}\sigma(X)-Y^T\vert\vert_2$ subject to the condition that the rows of $A$ must sum to one. Here $\sigma(X)$ denotes an $N$-dimensional vector formed of the standard deviation of the rows of $X$.  $AX^T+\vec{b}\sigma(X)$ is equivalent to augmenting $X$ by appending $\sigma(X)$ to $X$ as an additional column and then fitting a $(T+1)\times L$ matrix and no bias. Having appended this columns we then proceed along the same lines as before; subtracting the row means from $X$ and $Y$ and solving the resulting regression problem in closed form. One should be careful not to change the final column in this process and not include $\sigma(x)$ in the computation of the mean.

\section{Experiment Details}\label{sec:experiment_details}

\textbf{Datasets:} For our experiments in Section~\ref{exp:performance} we use 8 standard time series benchmarking datasets: ETTh1 and ETTh2: 7-channel hourly datasets (Train-Val-Test Splits [8545,2881,2881]). Their per-minute equivalents; ETTm1, ETTm2 (also 7-channel) (Train-Val-Test Splits [34465,11521,11521]). ECL, an hourly 321-channel Electricity dataset (Train-Val-Test Splits [18317,2633,5261]), Weather, a per-10-minute resolution 21-channel weather dataset (Train-Val-Test Splits [36792,5271,10540]), Traffic; an 862-channel traffic dataset (Train-Val-Test Splits [12185,1757,3509]) and Exchange: a small 8-channel finance dataset (Train-Val-Test Splits [5120,665,1422]).

In each case we use the well-established dataset divisions and normalisation protocols. We refer the reader to \citep{wu2021autoformer} for further details.

\textbf{Models:} The models we compare are DLinear, NLinear, RLinear, FITS and Linear (a single linear layer neural network). We also run FITS+IN and DLinear+IN. FITS+IN corresponds to the implementation of FITS used in \citet{xu2023fits}. Alongside these we run the closed-form solutions (OLS and OLS+IN). The mathematics behind these solutions are explained in Sec~\ref{sec:closed_form}. These are implemented using the LinearRegression model from scikit-learn using an SVD solver.

\textbf{Hyperparameters:} For each model, dataset, and horizon combination we train for 50 epochs using a learning rate of 0.0005 and the Adam optimizer with the default hyperparameter settings. We use a batch size of 128 in all experiments. We track the validation loss during training. At test time we load the model with minimal validation loss to evaluate on the training set, which is equivalent to early stopping. Each experiment is run (at least) 3 times using different random seeds and the standard deviation of the MSEs is computed and given in Table~\ref{tab:comparison}. We test on prediction horizons of 96, 192, 336 and 720 which are the standard in the literature \citep{nie2022time}. In all cases we use a context length of 720, as per the setting used by \citet{xu2023fits}. Our implementation of the DLinear model is taken from \citet{DlinearGITHUB}. Our implementation of FITS is taken from \citet{FITSGITHUB}. We re-implemented the RLinear and NLinear models using the detailed descriptions of these models in their respective papers \citep{zeng2023transformers, li2023revisiting}. 

\textbf{Weight Comparison Experiments: } In Section~\ref{sec:experiments} we compare the weight matrices, biases and forecasts of the different models. The hyperparameter settings are largely identical to those used to populate Table~\ref{tab:comparison}. One difference is that we compare our weights/biases/forecasts at the end of 50 epochs of training rather than using early stopping. A second difference is that the Figure \ref{fig:cosines} shows the cosine similarity over 350 training epochs and uses a learning rate of 0.0002 rather than 0.0005. The purpose of this change was to demonstrate clearly the convergence behaviour of these models, which inevitably requires a longer training run. All figures are obtained after training on the ETTh1 dataset. For the weight, forecast and cosine similarity figures (Figures~\ref{fig:heatmaps}, \ref{fig:forecasts}, \ref{fig:cosines}) we use a prediction horizon of 336, The bias figure (Figure~\ref{fig:biases}) uses a horizon of 720. 

\section{Further Discussion}\label{sec:further_discussion}

\subsection{Extracting the Weight Matrices}\label{sec:extracting_matrices}
Suppose that we have a trained model of the form $f(\vec{x})=A\vec{x}+\vec{b}\sigma(x)$ and we wish to determine $A$ and $\vec{b}$. The vector of all ones has standard deviation equal to zero. Therefore passing in this vector we obtain $f(\vec{1}) = A\vec{1} = \sum_{i=1}^L A_{ji}$, i.e the sum of the columns of $A$. Let $\frac{L}{\sqrt{L-1}}\vec{e_i}$ be a multiple of the $i$\textsuperscript{th} coordinate vector $\vec{e_i}$, where the multiple is chosen so that the vector has standard deviation equal to one. Passing in this vector for $f$ we get:
\begin{align}
    f(\vec{e_i}) = \frac{L}{\sqrt{L-1}}A\vec{e_i}+\vec{b}\sigma\bigg(\frac{L}{\sqrt{L-1}}\vec{e_i}\bigg) = \frac{L}{\sqrt{L-1}}A\vec{e_i}+\vec{b} = \frac{L}{\sqrt{L-1}}A_{\cdot, i} +\vec{b} \label{eqn:extract_weights}
\end{align}
One may solve this system of equations to derive $A$ and $\vec{b}$. In particular, $\sum_{i=1}^L f(\vec{e_i}) = L \vec{b} + \frac{L}{\sqrt{L-1}}\sum_{i=1}^LA_{\cdot, i}$.  So:
\begin{align*}
  \bigg(\frac{\sqrt{L-1}}{L} \sum_{i=1}^L f(\vec{e_i})\bigg) - f(\vec{1}) = (\sqrt{L-1})\vec{b}
\end{align*}
Having obtained $\vec{b}$ one can then use Eqn.~\ref{eqn:extract_weights} to derive the columns of $A$.

\subsection{Limitations and Future Work}\label{sec:limitations}

In Section~\ref{sec:model_analysis} we show how Linear+IN and Linear+RevIN have the same model classes. While Linear+RevIN and Linear+IN are identical in the single channel setting, they can differ subtly in the multi-channel setting. Specifically, in the setting where one shares weights of the linear layer across channels, but allows RevIN separate affine parameters per channel, then RevIN can have marginally different biases for each channel.

We wish to reiterate that our findings for FITS hold when the low-pass filter is not applied. As we have said, we are motivated to understand each model under their optimal settings, as such we have ignored the LPF which typically hinders performance. When one uses an LPF there will typically be restrictions on the model class meaning it is not equivalent to unconstrained linear regression. A further analysis of this is required in future work. 

One of the key contributions of FITS (\cite{xu2023fits}) is that it allows one to compress models by disregarding higher frequencies during training. Having established how to map between FITS models and their underlying affine representations, this opens the possibility of using the FITS technique to compress OLS solutions post hoc. This is something which should be looked at in future work.

\end{document}